\documentclass[letterpaper]{article} 
\usepackage{aaai24}
\usepackage{times}
\usepackage{helvet}
\usepackage{courier}
\usepackage[hyphens]{url}
\usepackage{graphicx} 
\usepackage{natbib}
\usepackage{caption} 
\title{ 
Learning Domain-Independent Heuristics for Grounded and Lifted Planning
}

\author{ 
  Dillon Z. Chen,\textsuperscript{\rm 1,2}
  Sylvie Thi\'ebaux,\textsuperscript{\rm 1,2}
  Felipe Trevizan\textsuperscript{\rm 1}\\
} 
\affiliations{ 
\textsuperscript{\rm 1}School of Computing, The Australian National University,\\
\textsuperscript{\rm 2}LAAS-CNRS, Universit\'e de Toulouse\\
\{dillon.chen, sylvie.thiebaux\}@laas.fr,
felipe.trevizan@anu.edu.au
}

\usepackage[utf8]{inputenc}
\usepackage{amssymb,amsthm,amsmath}
\usepackage{enumitem}
\usepackage{tikz}
\usepackage{caption}
\usepackage{subcaption}
\usepackage{pdflscape}
\usepackage{afterpage}
\usepackage{tabularx}
\usepackage{longtable}
\usepackage{centernot}
\usepackage{complexity}
\usepackage[ruled,vlined,linesnumbered]{algorithm2e} 
\usepackage{mathrsfs}
\usepackage{multirow}
\usepackage{lipsum}
\usepackage{array}
\usepackage{booktabs}
\usepackage{tikz-cd}
\usepackage{colortbl}
\usepackage{svg}

\DeclareMathOperator*{\aggr}{agg}
\DeclareMathOperator*{\comb}{cmb}
\DeclareMathOperator*{\pe}{IF}

\DeclareMathOperator*{\pre}{pre}

\DeclareMathOperator*{\add}{add}
\DeclareMathOperator*{\varval}{var:val}

\DeclareMathOperator*{\del}{del}
\DeclareMathOperator*{\eff}{eff}

\DeclareMathOperator*{\val}{val}
\DeclareMathOperator*{\goal}{goal}
\DeclareMathOperator*{\true}{true}

\DeclareMathOperator*{\ff}{FF}

\newclass{\N}{N}
\newclass{\coNTIME}{coNTIME}
\newclass{\coNSPACE}{coNSPACE}
\newclass{\coNPSPACE}{coNPSPACE}
\newclass{\EXPTIME}{EXPTIME}
\newclass{\NEXPTIME}{NEXPTIME}
\newclass{\coNEXPTIME}{coNEXPTIME}
\newclass{\NEXPSPACE}{NEXPSPACE}
\newclass{\coNEXPSPACE}{coNEXPSPACE}
\newclass{\ASPACE}{ASPACE}
\newclass{\ATIME}{ATIME}
\newclass{\APSPACE}{APSPACE}
\newclass{\AEXPTIME}{AEXPTIME}
\newclass{\AEXPSPACE}{AEXPSPACE}

\newcolumntype{Y}{>{\centering\arraybackslash}X}

\newcommand{\predicates}{\mathcal{P}}
\newcommand{\objects}{\mathcal{O}}
\newcommand{\actionschema}{\mathcal{A}}

\newcommand{\probstrips}{\ensuremath\gen{P, A, s_0, G}}

\newcommand{\probfdr}{\ensuremath\gen{\mathcal{V}, A, s_0, s_\star}}
\newcommand{\problifted}{\ensuremath\gen{\predicates, \objects, \actionschema, s_0, G}}

\newcommand{\hadd}{\ensuremath{h^{\add}}}
\newcommand{\hsum}{\ensuremath{h^{\add}}}
\newcommand{\hmax}{\ensuremath{h^{\max}}}

\newcommand{\hplus}{\ensuremath{h^{+}}}
\newcommand{\hopt}{\ensuremath{h^{*}}}
\newcommand{\hff}{\ensuremath{h^{\ff}}}

\newcommand{\blind}{\text{blind}}
\newcommand{\hfftable}{$h^{\text{FF}}$}

\newcommand{\shgn}{\text{HGN}}

\newcommand{\existingfont}{}
\newcommand{\graphrepfont}{}

\newcommand{\asg}{\existingfont{ASG}}
\newcommand{\pdg}{\existingfont{PDG}}
\newcommand{\slg}{\graphrepfont{SLG}}
\newcommand{\dlg}{\graphrepfont{DLG}}
\newcommand{\flg}{\graphrepfont{FLG}}

\newcommand{\llg}{\graphrepfont{LLG}}

\usetikzlibrary{fit}

\newtheorem{theorem}{Theorem}[section]

\theoremstyle{definition}
\newtheorem{definition}[theorem]{Definition}

\def\N{{\mathbb N}}
\def\Ne{{\mathcal N}}

\def\h{{\mathbf h}}

\def\X{{\mathbf{X}}}

\def\y{{\boldsymbol{y}}}

\def\params{{\Theta}}

\def\Graph{\mathcal{G}}
\def\d{\delta}

\def\e{\varepsilon}

\def\F{{\mathcal{F}}}

\def\R{\mathbb{R}}

\def\S{\mathcal{S}}

\def\la{\leftarrow}

\renewcommand{\phi}{\varphi}

\newcommand{\abs}[1]{\left| #1 \right|}
\newcommand{\norm}[1]{\left\| #1 \right\|}
\newcommand{\gen}[1]{\left< #1 \right>}
\newcommand{\set}[1]{\left\{ #1 \right\}}
\newcommand{\seta}[1]{\{ #1 \}}

\newcommand{\setbig}[1]{\bigl\{ #1 \bigr\}}

\newcommand{\lr}[1]{\left( #1 \right)}
\newcommand{\biglr}[1]{\bigl( #1 \bigr)}

\newcommand{\todo}[1]{{\color{red}#1}}

\newcommand{\graph}[1]{\gen{V#1,E#1,\X#1}}
\newcommand{\edge}[2]{\gen{#1}_{#2}}
\newcommand{\neighbours}{\mathcal{N}}

\DeclareMathOperator*{\concat}{%
    \mathchoice%
        {\Big\Vert}%
        {\big\Vert}%
        {\Vert}%
        {\Vert}%
}

\DeclareMathOperator*{\concatsmall}{%
        {\Vert}%
}

\newcommand{\yyshift}{}
\newcommand{\xxshift}{}
\newcommand{\csize}{}

\definecolor{caribbeangreen}{rgb}{0.0, 0.8, 0.6}
\definecolor{brilliantlavender}{rgb}{0.96, 0.73, 1.0}
\definecolor{amethyst}{rgb}{0.6, 0.4, 0.8}
\definecolor{ao(english)}{rgb}{0.0, 0.5, 0.0}
\definecolor{arylideyellow}{rgb}{0.91, 0.84, 0.42}
\definecolor{asparagus}{rgb}{0.53, 0.66, 0.42}
\definecolor{aquamarine}{rgb}{0.5, 1.0, 0.83}
\definecolor{babyblue}{rgb}{0.54, 0.81, 0.94}
\definecolor{fwtchanged}{rgb}{0.3, 0.3, 0.7}

\newcommand{\header}[1]{\rotatebox[origin=l]{90}{\hspace*{-0.2cm} #1}}

\newcommand{\colorofcell}{blue}
\newcommand{\first}[1]{\cellcolor{\colorofcell!30}{\textbf{#1}}}
\newcommand{\second}[1]{\cellcolor{\colorofcell!20}{{#1}}}
\newcommand{\third}[1]{\cellcolor{\colorofcell!10}{{#1}}}
\newcommand{\normalcell}[1]{{{#1}}}
\newcommand{\zerocell}[1]{-}

\renewcommand{\todo}[1]{}

\begin{document}

\maketitle

\begin{abstract}
We present three novel graph representations of planning tasks suitable for learning domain-independent heuristics using Graph Neural Networks (GNNs) to guide search.
In particular, to mitigate the issues caused by large grounded GNNs we present the first method for learning domain-independent heuristics with only the lifted representation of a planning task.
We also provide a theoretical analysis of the expressiveness of our models, showing that some are more powerful than STRIPS-HGN, the only other existing model for learning domain-independent heuristics.
Our experiments show that our heuristics generalise to much larger problems than those in the training set, vastly surpassing STRIPS-HGN heuristics.
\end{abstract}

\section{Introduction}
Graph Neural Networks (GNNs) have recently attracted the interest of the planning community, for learning heuristic cost estimators, task orderings, value functions, action policies, and portfolios, to name a few \cite{shen20stripshgn,garg2020symnet,karia2021learning,staahlberg2022optimalpolicies,ma2020online,sharma2022symnet2,teichteil:etal:23}.
GNNs exhibit great generalisation potential, since once trained, they offer outputs for any graph, regardless of size or structure.
Representing the structure of planning domains as graphs, GNNs can train on a set of small problems to learn generalised policies and heuristics that apply to all problems in a domain.
As noted by~\citet{shen20stripshgn}, this also allows for learning heuristics applicable to multiple domains, or even domain-independent heuristics that apply to domains unseen during training.

In this paper, we explore the use of GNNs for learning both domain-dependent and domain-independent heuristics for classical planning, with an emphasis on the latter.
To the best of our knowledge, STRIPS-HGN \cite{shen20stripshgn} is the only existing model designed to learn domain-independent heuristic functions from scratch.
The models in \cite{staahlberg2022optimalpolicies} are inherently domain-dependent, given that they use different update functions for predicates of the planning problems, and hence cannot generalise to unseen problems with a different number of predicates.
Neural Logic Machines \cite{dong2019nlm,gehring2022reinforcement} are also domain-dependent models as they assume a maximum arity of input predicates.

STRIPS-HGN has several drawbacks when learning domain-independent heuristics: (1) its hypergraph representation of planning tasks ignores delete lists and thus cannot theoretically learn \hopt, (2) its aggregation function is not permutation invariant due to ordering the neighbours of each node which may prevent it from generalising effectively, (3) it assumes a bound on the sizes of action preconditions and effects, meaning that it also has to discard certain edges in its hypergraph in its message updating step, and (4) it requires constructing the whole grounded hypergraph, whose size becomes impractical for large problems.

Our contributions remedy these issues and make the following advances to the state of the art.
Building on well-known planning formalisms, namely propositional STRIPS, FDR, and lifted STRIPS, we define novel grounded and lifted graph representations of planning tasks suitable for learning domain-independent heuristics.
In particular, this results in the first domain-independent GNN heuristic based on a lifted graph representation.
We also establish the theoretical expressiveness of Message-Passing Neural Networks (MPNN) acting upon our graphs in terms of the known domain-independent heuristics they are able to learn, and suggest further research directions for learning \hopt.

We then conduct two sets of experiments to complement our theory and evaluate the effectiveness of learned heuristics.
The first set aims to measure the accuracy of learned heuristics on unseen tasks, while the second set evaluates the effectiveness of such learned heuristics for heuristic search.
Planners guided by heuristics learnt using our new graphs solve significantly larger problems than those considered by \citet{shen20stripshgn}, \citet{karia2021learning} and \citet{stahlberg:etal:23:kr}.
In the domain-dependent setting, planners guided by our lifted heuristics achieves greater coverage than using \hff{} in several domains and returns lower cost plans overall.

\section{Background and Notation}
\subsubsection{Planning}
A classical planning task~\cite{geffner2013concise} is a state transition model $\Pi = \langle S, A, s_0, G \rangle$ consisting of a set $S$ of states, a set $A$ of actions, an initial state $s_0$, and a set $G$ of goal states.
Each action $a\in A$ is a function $a: S \to S \cup{\bot}$ mapping a state $s$ in which the action is applicable to its successor $a(s)$, and states in which it is not applicable to $\bot$.
The cost of the action is $c(a)\in \N$.
A solution or a \emph{plan} in this model is a sequence of actions $\pi=a_1,\ldots,a_n$ such that $s_i = a_i(s_{i-1}) \neq \bot$ for all $i\in \{1, \dots, n\}$ and $s_n \in G$.
In other words, a plan is a sequence of applicable actions which when executed, progresses our initial state to a goal state.
The cost of $\pi$ is $c(\pi) = \sum_{i=1}^n c(a_i)$.
A planning task is \emph{solvable} if there exists at least one plan.
We now describe three ways to represent planning tasks.

A \emph{STRIPS planning task} is a tuple $\Pi = \probstrips$ with $P$ a set of propositions (or facts), $A$ a set of actions, $s_0 \subseteq P$ an initial state, and $G \subseteq P$ the goal condition.
A state $s$ is a subset of $P$ and is a goal state if $G \subseteq s$.
An action $a \in A$ is a tuple $\gen{\pre(a), \add(a), \del(a)}$ with $\pre(a), \add(a), \del(a) \subseteq P$ and $\add(a) \cap \del(a) = \emptyset$, and has an associated cost $c(a) \in \R$.
The action is applicable in a state $s$ if $\pre(a) \subseteq s$, and leads to the successor state $s' = (s\setminus \del(a)) \cup \add(a)$.

An \emph{FDR planning task}~\citep{helmert2009concise} is a tuple $\Pi = \probfdr$ where $\mathcal{V}$ is a finite set of state variables $v$, each with a finite domain $D_v$.
A fact is a pair $\gen{v, d}$ where $v \in \mathcal{V}, d \in D_v$.
A partial variable assignment is a set of facts where each variable appears at most once.
A total variable assignment is a partial variable assignment where each variable appears once.
The initial state $s_0$ is a total variable assignment and the goal condition $s_\star$ is a partial variable assignment.
Again, $A$ is a set of actions of the form $a = \gen{\pre(a), \eff(a)}$ where $\pre(a)$ and $\eff(a)$ are partial variable assignments.
An action $a$ is applicable in $s$ if $\pre(a) \subseteq s$, and leads to the successor state $s' = (s \cup \eff(a)) \setminus \seta{\gen{v, d} \in s \mid \exists d' \in D_v, \gen{v, d'} \in \eff(a) \wedge d\not=d'}.$

A \emph{lifted planning task}~\cite{lauer2021polytime} is a tuple $\Pi = \problifted$ where $\predicates$ is a set of first-order predicates, $\actionschema$ is a set of action schema, $\objects$ is a set of objects, $s_0$ is the initial state and $G$ is the goal condition.
A predicate $P \in \predicates$ has parameters $x_1,\ldots,x_{n_P}$ for $n_P \in \N$, noting that $n_P$ depends on $P$ and it is possible for a predicate to have no parameters.
A predicate with $n$ parameters is an $n$-ary predicate.
A predicate can be instantiated by assigning some of the $x_i$ with objects from $\objects$ or other defined variables.
A predicate where all variables are assigned with objects is grounded, and is called a ground proposition.
The initial state and goal condition are sets of ground propositions.
An action schema $a \in \actionschema$ is a tuple $\gen{\Delta(a), \pre(a), \add(a), \del(a)}$ where $\Delta(a)$ is a set of parameter variables and $\pre(a), \add(a)$ and $\del(a)$ are sets of predicates from $\predicates$ instantiated with either parameter variables or objects in $\Delta(a) \cup \objects$.
Similarly to predicates, an action schema with $n=\abs{\Delta(a)}$ parameter variables is an $n$-ary action schema.
An action schema where each variable is instantiated with an object is an action.
Action application and successor states are defined in the same way for both STRIPS and lifted planning.

\subsubsection{Graph neural networks}
The introduction of graph neural networks (GNN) requires additional terminology.
In the context of learning tasks, we define a graph with edge labels to be a tuple $\graph{}$ where $V$ is a set of nodes, $E$ a set of undirected edges with labels where $\edge{v,u}{\iota}=\edge{u,v}{\iota}\in E$ denotes an undirected edge with label $\iota$ between nodes $u,v \in V$, and $\X: V \to \R^d$ a function representing the node features of the graph.
The edge neighbourhood of a node $u$ in a graph under edge label $\iota$ is $\neighbours_\iota(u) = \set{\edge{u,v}{\kappa} \in E \mid \kappa = \iota}$.
The edge neighbourhood of a node $u$ in a graph is $\neighbours(u) = \bigcup_{\iota \in \mathcal{R}}\neighbours_\iota(u)$ where $\mathcal{R}$ is the set of edge labels for the graph.

A Message-Passing Neural Network (MPNN)~\cite{pmlr-v70-gilmer17a} is a type of GNN which iteratively updates node embeddings of a graph with edge labels locally in one-hop neighbourhoods with the general message passing equation
\begin{align*}
\h_u^{(t+1)} = 
\textstyle
\comb^{(t)}\biglr{
\h_u^{(t)}, 
\textstyle\aggr^{(t)}_{\edge{u,v}{\iota} \in \Ne(u)}
f^{(t)}\biglr{\h_u^{(t)}, \h_v^{(t)}, \iota}
}
\end{align*}
where in the $t$-th iteration or layer of the network, $\h_u^{(t)} \in \R^{F^{(t)}}$ is the embedding of node $u$ of dimension $F^{(t)}$, and $\h_u^{(0)}$ is given by the node feature corresponding to $u$ in $\X$.
We have that $\comb^{(t)}$ and $f^{(t)}$ are arbitrary almost everywhere differentiable functions and {$\aggr^{(t)}$} is usually a differentiable permutation invariant function acting on sets of vectors such as sum, mean or component-wise max.

In order for an MPNN to produce a graph representation for an input, it is then common to pool all the node embeddings after a number of message passing updates with a \emph{graph readout} function $\Phi$ which is again usually given by a differentiable permutation invariant function.

\section{Representation}

In this section, we introduce three novel graph representations designed for learning heuristic functions for planning tasks.
Each graph is tailored to a specific task representation and all of them allow us to learn \emph{domain-independent} heuristic functions.

\subsection{Grounded Graphs}
A graph representation for grounded STRIPS problems already exists, namely the STRIPS problem description graph (\pdg{})~\citep{shleyfman2015heuristics}.
It was originally used to study which classical heuristics were invariant under symmetries in the planning task.
In order to learn heuristics, we provide an alternative graph representation for STRIPS problems which includes node features and edge labels.

\begin{definition}
The \emph{STRIPS learning graph (\slg{})} of a STRIPS problem $\probstrips$ is the graph $\graph{}$ with
\begin{itemize}
\item $V = A \cup P$, 
\item $E = E_{\pre} \cup E_{\add} \cup E_{\del}$ where for $\iota \in \set{\pre, \add, \del}$ 
$$E_{\iota} = \set{\edge{a,p}{\iota} \mid p \in \iota(a), a \in A},$$ 
\item $\X: V \to \R^3$ defined by $u \mapsto [u \in P; u \in s_0; u \in G].$
\end{itemize}
\end{definition}
By allowing for edge labels, we only require one node for each proposition to encode the semantics of action effects in contrast to STRIPS \pdg{} which requires three nodes for each proposition.
Thus \slg{} is smaller while not losing any information.
We also note that three dimensional node features are sufficient for encoding whether a node corresponds to an action or proposition, and in the latter case, whether it is true in the initial state and present in the goal.
Fig.~\ref{fig:slg} illustrates an example \slg{} subgraph.

\renewcommand{\xxshift}{2cm}
\newcommand{\limmmmmmme}{0.65cm}
\renewcommand{\yyshift}{\limmmmmmme}
\renewcommand{\csize}{0.52cm}
\newcommand{\ang}{360/7}
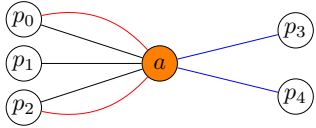
\begin{figure}
  \centering
    \resizebox{0.5\columnwidth}{!}{%
    \begin{tikzpicture}[square/.style={regular polygon,regular polygon sides=4}]
    \node[shape=circle,draw=black,minimum size=\csize,fill=orange] (A) at (0,0) {};
    \node[shape=circle,draw=black,minimum size=\csize] (p0) at ([shift={(-\xxshift,1*\yyshift)}]A) {};
    \node[shape=circle,draw=black,minimum size=\csize] (p1) at ([shift={(-\xxshift,0*\yyshift)}]A) {};
    \node[shape=circle,draw=black,minimum size=\csize] (p2) at ([shift={(-\xxshift,-1*\yyshift)}]A) {};
    \node[shape=circle,draw=black,minimum size=\csize] (p3) at ([shift={(\xxshift,0.75*\yyshift)}]A) {};
    \node[shape=circle,draw=black,minimum size=\csize] (p4) at ([shift={(\xxshift,-0.75*\yyshift)}]A) {};
    \node at (A) {$a$};
    \node at (p0) {$p_0$};
    \node at (p1) {$p_1$};
    \node at (p2) {$p_2$};
    \node at (p3) {$p_3$};
    \node at (p4) {$p_4$};

    \path [-,draw=black] (p0) edge (A);
    \path [-,draw=black] (p1) edge (A);
    \path [-,draw=black] (p2) edge (A);
    \path [-,draw=blue] (p3) edge (A);
    \path [-,draw=blue] (p4) edge (A);
    \path [-,draw=red] (p0) edge[bend left] (A);
    \path [-,draw=red] (p2) edge[bend right] (A);
  \end{tikzpicture}
}
  \caption{The \slg{} subgraph of an action $a$ defined by $\pre(a) = \set{p_0,p_1,p_2}$, $\add(a) = \set{p_3,p_4}$ and $\del(a) = \set{p_0,p_2}$, indicated by black, blue and red edges respectively. 
  }\label{fig:slg}
\end{figure}

The FDR problem description graph (\pdg{})~\citep{pochter2011exploiting} is an existing graph representation designed to identify symmetrical states during search for FDR problems.
Since \pdg{} is not designed for learning, it lacks vector node features and edge labels.
Def.~\ref{def:flg} extends FDR \pdg{} with learning in mind by retaining its graph structure and adding node features and edge labels.
Fig.~\ref{fig:flg} illustrates an example of an \flg{} subgraph.

\renewcommand{\xxshift}{0.75cm}
\renewcommand{\yyshift}{-\limmmmmmme}
\newcommand{\flgfontsize}{\scriptsize}
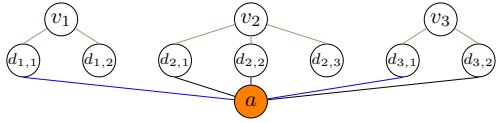
\begin{figure}
  \centering
    \resizebox{0.8\columnwidth}{!}{%
    \begin{tikzpicture}[square/.style={regular polygon,regular polygon sides=4}]
    \node[shape=circle,draw=black,minimum size=\csize,fill=orange] (A) at (0,\yyshift) {};
    \node at (A) {$a$};
    \node[shape=circle,draw=black,minimum size=\csize] (V1) at ([shift={(-4*\xxshift,-2*\yyshift)}]A) {};
    \node at (V1) {$v_1$};
    \node[shape=circle,draw=black,minimum size=\csize] (V2) at ([shift={(0*\xxshift,-2*\yyshift)}]A) {};
    \node at (V2) {$v_2$};
    \node[shape=circle,draw=black,minimum size=\csize] (V3) at ([shift={(4*\xxshift,-2*\yyshift)}]A) {};
    \node at (V3) {$v_3$};
    \node[shape=circle,draw=black,minimum size=\csize] (D11) at ([shift={(-0.8*\xxshift,\yyshift)}]V1) {};
    \node at (D11) {\flgfontsize$d_{1,1}$};
    \node[shape=circle,draw=black,minimum size=\csize] (D12) at ([shift={(0.8*\xxshift,\yyshift)}]V1) {};
    \node at (D12) {\flgfontsize$d_{1,2}$};
    \node[shape=circle,draw=black,minimum size=\csize] (D21) at ([shift={(-1.6*\xxshift,\yyshift)}]V2) {};
    \node at (D21) {\flgfontsize$d_{2,1}$};
    \node[shape=circle,draw=black,minimum size=\csize] (D22) at ([shift={(0*\xxshift,\yyshift)}]V2) {};
    \node at (D22) {\flgfontsize$d_{2,2}$};
    \node[shape=circle,draw=black,minimum size=\csize] (D23) at ([shift={(1.6*\xxshift,\yyshift)}]V2) {};
    \node at (D23) {\flgfontsize$d_{2,3}$};
    \node[shape=circle,draw=black,minimum size=\csize] (D31) at ([shift={(-0.8*\xxshift,\yyshift)}]V3) {};
    \node at (D31) {\flgfontsize$d_{3,1}$};
    \node[shape=circle,draw=black,minimum size=\csize] (D32) at ([shift={(0.8*\xxshift,\yyshift)}]V3) {};
    \node at (D32) {\flgfontsize$d_{3,2}$};

    \path [-,draw=asparagus] (V1) edge (D11.north);
    \path [-,draw=asparagus] (V1) edge (D12.north);
    \path [-,draw=asparagus] (V2) edge (D21.north);
    \path [-,draw=asparagus] (V2) edge (D22.north);
    \path [-,draw=asparagus] (V2) edge (D23.north);
    \path [-,draw=asparagus] (V3) edge (D31.north);
    \path [-,draw=asparagus] (V3) edge (D32.north);

    \path [-,draw=black] (A) edge (D21.south);
    \path [-,draw=black] (A) edge (D32.south);

    \path [-,draw=blue] (A) edge (D11.south);
    \path [-,draw=blue] (A) edge (D22.south);
    \path [-,draw=blue] (A) edge (D31.south);
  \end{tikzpicture}
    }
  \caption{The \flg{} subgraph of an action $a$ defined by $\pre(a) = \set{\gen{v_2,d_{2,1}},\gen{v_3,d_{3,2}}}$ and $\eff(a) = \{\gen{v_1,d_{1,1}}$, $\gen{v_2,d_{2,2}}$, $\gen{v_3,d_{3,1}}\}$, indicated by black and blue edges respectively. Asparagus edges link variables and values.
  }\label{fig:flg}
\end{figure}
\begin{definition}\label{def:flg} The \emph{FDR learning graph (\flg{})} of an FDR problem
$\probfdr$ is the graph $\graph{}$ with
\begin{itemize}
\item $V = \mathcal{V} \cup \bigcup_{v \in \mathcal{V}} D_v \cup A$,
\item $E = E_{\varval} \cup E_{\pre} \cup E_{\eff}$ with
\begin{align*}
E_{\varval} &= \textstyle\bigcup_{v \in \mathcal{V}}\seta{
\edge{v, d}{\varval}
\mid d \in D_v} \\
E_{\pre} &= \textstyle\bigcup_{a \in A}\seta{
\edge{d, a}{\pre}
\mid (v, d) \in \pre(a)} \\
E_{\eff} &= \textstyle\bigcup_{a \in A}\seta{
\edge{d, a}{\eff}
\mid (v, d) \in \eff(a)},
\end{align*}
\item $\X: V \to \R^5$ defined by 
\begin{align*}
u \mapsto [u\in\mathcal{V};u\in A;\val(u); \true(u); \goal(u)]
\end{align*}
where $\val(u) = \exists v \in \mathcal{V}$, $u \in D_v$, $\true(u)= \exists v \in \mathcal{V}$,
$\gen{v, u} \in s_0$ and $\goal(u) = \exists v \in \mathcal{V}$, $\gen{v, u} \in s_{\star}$.
\end{itemize}
\end{definition}

\subsection{Lifted Graphs}

Lifted algorithms for planning offer an advantage by avoiding the need for grounding thus saving both time and memory.
To leverage these benefits for heuristic learning, a lifted graph representation that is amenable to learning is needed.
However, designing such graphs is non-trivial due to the extra relations to encode, namely the interactions between predicates, action schema, propositions true in the current state, the goal condition and objects.
The only graph representation encoding all the information of a lifted planning task is the \emph{abstract structure graph} (\asg{})~\citep{sievers2019asg}.
Similarly to \pdg{}, \asg{} was designed to compute symmetries but it has also been used for learning planning portfolios~\citep{katz2018delfi}.

 An \asg{} is constructed by first defining a coloured graph on \emph{abstract structures}, a recursive structure defined with sets, tuples, and the input objects, and then defining a lifted planning task as an abstract structure.
\asg{}s have several limitations when used with MPNNs for making predictions.
Their encoding of predicate and action schema arguments is done via a sequence or directed path, where the graph uses a directed path of length $n$ to encode $n$ arguments.
There are also many more auxiliary nodes to encode the abstract structures.
These issues cause problems as a larger receptive field is required for MPNNs to learn the structure and semantics of the planning problem, and directed edges limit information flow and expressiveness when used with MPNNs.

\renewcommand{\yyshift}{0.75cm}
\newcommand{\yyshiftb}{0.825cm}
\renewcommand{\xxshift}{0.9cm}
\renewcommand{\csize}{0.6cm}
\newcommand{\descccsize}{\scriptsize}
\newcommand{\desccc}{5.0cm}
\newcommand{\desccccc}{0.5cm}
\newcommand{\anchor}{west}
\newcommand{\llgword}[1]{\tiny\text{#1}}
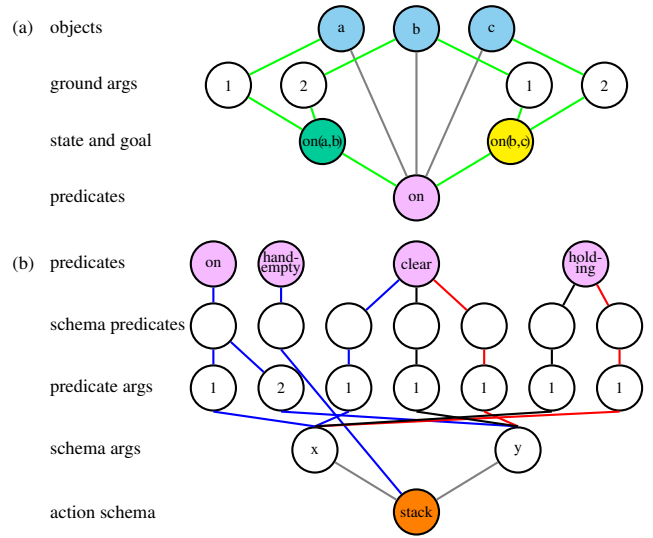
\begin{figure}
  \begin{subfigure}[b]{\linewidth}
    \begin{tikzpicture}[thick]
      \node[shape=circle,draw=black,minimum size=\csize,fill=brilliantlavender] (on) at (0,0) {
      };
      \node[shape=circle,draw=black,minimum size=\csize,fill=caribbeangreen] (onab) at (-2.5/1.8*\xxshift,\yyshift) {
      };
      \node[shape=circle,draw=black,minimum size=\csize,fill=yellow] (onbc) at ( 2.5/1.8*\xxshift,\yyshift) {
      };
      \node[shape=circle,draw=black,minimum size=\csize] (ab1) at ( -5/1.8*\xxshift,2*\yyshift) {\llgword 1};
      \node[shape=circle,draw=black,minimum size=\csize] (ab2) at ( -3/1.8*\xxshift,2*\yyshift) {\llgword 2};
      \node[shape=circle,draw=black,minimum size=\csize] (bc1) at (  3/1.8*\xxshift,2*\yyshift) {\llgword 1};
      \node[shape=circle,draw=black,minimum size=\csize] (bc2) at (  5/1.8*\xxshift,2*\yyshift) {\llgword 2};
      \node[shape=circle,draw=black,minimum size=\csize,fill=babyblue] (a) at (  -2/1.8*\xxshift,3*\yyshift) {
      };
      \node[shape=circle,draw=black,minimum size=\csize,fill=babyblue] (b) at (   0/1.8*\xxshift,3*\yyshift) {
      };
      \node[shape=circle,draw=black,minimum size=\csize,fill=babyblue] (c) at (   2/1.8*\xxshift,3*\yyshift) {
      };
      \node at (onab) {$\llgword{on(\!a,b\!)}$};
      \node at (onbc) {$\llgword{on(\!b,c\!)}$};
      \node at (a) {$\llgword{a}$};
      \node at (b) {$\llgword{b}$};
      \node at (c) {$\llgword{c}$};
      \node at (on) {$\llgword{on}$};

      \node[anchor=\anchor] (zzz) at (-\desccc-\desccccc,3*\yyshift) {\descccsize{(a)}};
      \node[anchor=\anchor] (zzz) at (-\desccc,3*\yyshift) {\descccsize{objects}};
      \node[anchor=\anchor] (zzz) at (-\desccc,2*\yyshift) {\descccsize{ground args}};
      \node[anchor=\anchor] (zzz) at (-\desccc,\yyshift) {\descccsize{state and goal}};
      \node[anchor=\anchor] (zzz) at (-\desccc,0) {\descccsize{predicates}};
  
      \path [-,draw=gray] (a) edge (on);
      \path [-,draw=gray] (b) edge (on);
      \path [-,draw=gray] (c) edge (on);

      \path [-,draw=green] (on) edge (onab);
      \path [-,draw=green] (on) edge (onbc);
      \path [-,draw=green] (onab) edge (ab1);
      \path [-,draw=green] (onab) edge (ab2);
      \path [-,draw=green] (onbc) edge (bc1);
      \path [-,draw=green] (onbc) edge (bc2);
      \path [-,draw=green] (a) edge (ab1);
      \path [-,draw=green] (b) edge (ab2);
      \path [-,draw=green] (b) edge (bc1);
      \path [-,draw=green] (c) edge (bc2);
    \end{tikzpicture}
  \end{subfigure}
  \renewcommand{\yyshift}{\yyshiftb}
  \begin{subfigure}[b]{\linewidth}
    \begin{tikzpicture}[thick]
      \node (empty) at (-2.5*\xxshift,1.5*\yyshift) {};

      \node[shape=circle,draw=black,minimum size=\csize,fill=brilliantlavender] (on) at (-3*\xxshift,\yyshift) {
      };
      \node[shape=circle,draw=black,minimum size=\csize,fill=brilliantlavender] (handempty) at (-2*\xxshift,\yyshift) {
      };
      \node[shape=circle,draw=black,minimum size=\csize,fill=brilliantlavender] (clear) at (0*\xxshift,\yyshift) {
      };
      \node[shape=circle,draw=black,minimum size=\csize,fill=brilliantlavender] (holding) at (2.5*\xxshift,\yyshift) {
      };

      \node at (on) {$\llgword{on}$};
      \node at ([shift={(0, 0.1)}]handempty) {$\llgword{hand-}$};
      \node at ([shift={(0, -0.075)}]handempty) {$\llgword{empty}$};
      \node at (clear) {$\llgword{clear}$};
      \node at ([shift={(0, 0.1)}]holding) {$\llgword{hold-}$};
      \node at ([shift={(0,-0.075)}]holding) {$\llgword{ing}$};

      \node[shape=circle,draw=black,minimum size=\csize] (l3) at (-3*\xxshift,0) {};
      \node[shape=circle,draw=black,minimum size=\csize] (l2) at (-2*\xxshift,0) {};
      \node[shape=circle,draw=black,minimum size=\csize] (l1) at (-1*\xxshift,0) {};
      \node[shape=circle,draw=black,minimum size=\csize] (00) at ( 0*\xxshift,0) {};
      \node[shape=circle,draw=black,minimum size=\csize] (r1) at ( 1*\xxshift,0) {};
      \node[shape=circle,draw=black,minimum size=\csize] (r2) at ( 2*\xxshift,0) {};
      \node[shape=circle,draw=black,minimum size=\csize] (r3) at ( 3*\xxshift,0) {};

      \node[shape=circle,draw=black,minimum size=\csize] (o01) at (-3*\xxshift,-\yyshift) {\llgword 1};
      \node[shape=circle,draw=black,minimum size=\csize] (o02) at (-2*\xxshift,-\yyshift) {\llgword 2};
      \node[shape=circle,draw=black,minimum size=\csize] (c01) at (-1*\xxshift
      ,-\yyshift) {\llgword 1};
      \node[shape=circle,draw=black,minimum size=\csize] (c11) at ( 0*\xxshift
      ,-\yyshift) {\llgword 1};
      \node[shape=circle,draw=black,minimum size=\csize] (c21) at ( 1*\xxshift,-\yyshift) {\llgword 1};
      \node[shape=circle,draw=black,minimum size=\csize] (h01) at ( 2*\xxshift,-\yyshift) {\llgword 1};
      \node[shape=circle,draw=black,minimum size=\csize] (h11) at ( 3*\xxshift,-\yyshift) {\llgword 1};

      \node[shape=circle,draw=black,minimum size=\csize] (x) at (-1.5*\xxshift,-2*\yyshift) {
      };
      \node[shape=circle,draw=black,minimum size=\csize] (y) at ( 1.5*\xxshift,-2*\yyshift) {
      };

      \node[shape=circle,draw=black,minimum size=\csize,fill=orange] (stack) at (0,-3*\yyshift) {
      };

      \node[anchor=\anchor] (zzz) at (-\desccc-\desccccc,\yyshift) {\descccsize{(b)}};
      \node[anchor=\anchor] (zzz) at (-\desccc,\yyshift) {\descccsize{predicates}};
      \node[anchor=\anchor] (zzz) at (-\desccc,0) {\descccsize{schema predicates}};
      \node[anchor=\anchor] (zzz) at (-\desccc,-\yyshift) {\descccsize{predicate args}};
      \node[anchor=\anchor] (zzz) at (-\desccc,-2*\yyshift) {\descccsize{schema args}};
      \node[anchor=\anchor] (zzz) at (-\desccc,-3*\yyshift) {\descccsize{action schema}};

      \node at (x) {$\llgword{x}$};
      \node at (y) {$\llgword{y}$};
      \node at (stack) {$\llgword{stack}$};
  
      \path [-,draw=blue] (on) edge (l3);
      \path [-,draw=blue] (l3) edge (o01);
      \path [-,draw=blue] (l3) edge (o02);
      \path [-,draw=blue] (o01.south) edge (x.north);
      \path [-,draw=blue] (o02.south) edge (y.north);
      \path [-,draw=blue] (handempty) edge (l2);
      \path [-,draw=blue] (l2.south) edge (stack);
      \path [-,draw=blue] (clear) edge (l1);
      \path [-,draw=blue] (c01) edge (l1);
      \path [-,draw=blue] (c01.south) edge (x.north);

      \path [-,draw=red] (holding) edge (r3);
      \path [-,draw=red] (r3) edge (h11);
      \path [-,draw=red] (h11.south) edge (x.north);
      \path [-,draw=red] (clear) edge (r1);
      \path [-,draw=red] (r1) edge (c21);
      \path [-,draw=red] (c21.south) edge (y.north);

      \path [-,draw=black] (holding) edge (r2);
      \path [-,draw=black] (r2) edge (h01);
      \path [-,draw=black] (h01.south) edge (x.north);
      \path [-,draw=black] (clear) edge (00);
      \path [-,draw=black] (00) edge (c11);
      \path [-,draw=black] (c11.south) edge (y.north);

      \path [-,draw=gray] (stack) edge (x);
      \path [-,draw=gray] (stack) edge (y);
    \end{tikzpicture}
  \end{subfigure}
  \caption{\llg{} instance subgraph (a) and schema subgraph (b) with graph layer descriptions of a Blocksworld instance.}\label{fig:llg}
\end{figure}

To overcome the issues with \asg{}s, Def.~\ref{def:llg} introduces a new graph representation for lifted planning tasks designed to be used with MPNNs.
It consists of two main components: the schema subgraph which encodes the domain's action schemas and the instance subgraph which encodes the instance specific information containing the current state and the goal condition.
We assume no partially instantiated action schema as is usual for most PDDL domains, but the definition could be extended to allow partial grounding as a tradeoff between expressiveness and graph size.

\newcommand{\llgitem}{\newline\noindent$\bullet$\;}
\begin{definition}\label{def:llg} Let $T \in \N$.
The \emph{lifted learning graph (\llg{})} of a
lifted problem $\Pi=\problifted$ is the graph $G=\graph{}$ with
\llgitem$V = \predicates \cup \objects \cup N(\actionschema) \cup N(s_0 \cup G)$ with
\begin{align*}
N(s_0 \cup G) &= 
\displaystyle\textstyle\bigcup_{p=P(o_1,\ldots,o_{n_P}) \in s_0 \cup G}
\set{p, p_1,\ldots,p_{n_P}} \\
N(\actionschema) &= \textstyle\bigcup_{a \in \actionschema} \bigl(
\set{a} \cup 
\set{a_{\d} \mid \d \in \Delta(a)} \cup 
\\
&\quad\;\;\textstyle\bigcup_{f\in \set{\pre,\add,\del}} 
\textstyle\bigcup_{p=P(\d_1,\ldots,\d_{n_P}) \in f(a)}
\\
&\quad\quad\bigl\{
p_{a,f}, 
p_{a,f,1}, \ldots, 
p_{a,f,n_P}
\bigr\}
\bigr) 
\end{align*}
$N(s_0 \cup G)$ contains nodes corresponding to the state and goal, and ground arguments layer as in Fig.~\ref{fig:llg}(a), while $N(\actionschema)$ provides all the nodes corresponding to the action schema, schema argument, predicate argument and schema predicate layers in Fig.~\ref{fig:llg}(b).
\llgitem$E = E_{\nu} \cup E_{\gamma} \cup \textstyle\bigcup_{f\in \set{\pre,\add,\del}} E_{f}$
where
\begin{align*}
E_{\nu} &= 
\setbig{
\edge{o, P}{\nu} \mid o \in \objects, P \in \predicates
} \;\cup \\
&\quad\;\;\setbig{
\edge{a, a_{\d}}{\nu} \mid \d \in \Delta(a), a \in \actionschema
} 
\\
E_{\gamma} &= \textstyle\bigcup_{p=P(o_1,\ldots,o_{n_P}) \in s_0 \cup G}
\bigl(
\setbig{
\edge{p, p_i}{\gamma} \!\mid\! i \!\in\! [n_P]}
\;\cup
\\
&\quad\;\;
\setbig{
\edge{p_i, o_i}{\gamma} \!\mid\! i \!\in\! [n_P]}
\cup 
\setbig{
\edge{p, P}{\gamma}
}
\bigr)
\\
E_f &= 
\textstyle\bigcup_{p=P() \in f(a)} 
\setbig{\edge{P, p_{a,f}}{f}, \edge{p_{a,f}, a}{f}}
\;\cup 
\\
&\quad\;\;\textstyle\bigcup_{p=P(\d_1,\ldots,\d_{n_P}) \in f(a), n_P \geq 1}
\bigl(
\setbig{\edge{P, p_{a,f}}{f}} \; \cup \\
&\quad\quad
\setbig{\edge{p_{a,f}, p_{a,f,i}}{f}, 
\edge{p_{a,f,i}, a_{\d_i}}{f} \!\mid\! i \!\in\! [n_P]}
\bigr)
\end{align*}
for $f \in \set{\pre, \add, \del}$.
$E_{\nu}$ connects objects to predicates and schemas to their arguments, as indicated by gray edges in Fig.~\ref{fig:llg}, $E_{\gamma}$ connects nodes in $\predicates$, $\objects$, and $N(s_0\cup G)$ in order to represent propositions in the goal and true in the state as instantiated predicates with objects in the correct arguments, and $\textstyle\bigcup_{f\in \set{\pre,\add,\del}} E_{f}$ connects nodes in $\predicates$ and $N(\actionschema)$ to encode the semantics of action schema in the graph.
\llgitem$\X: V \to \R^{5+T}$ defined by
\begin{align*}
u \mapsto [u \in \predicates; u \in \objects; u \in \actionschema; u \in s_0; u \in G] \concatsmall 
\overline{\pe}(u)
\end{align*}
where $\concatsmall$ denotes vector concatenation, $\overline{\pe}(u) = \pe(i)$ for $u$ of the
form $p_i$ or $p_{a,f,i}$ with $f \in \set{\pre,\add,\del}$ and $\overline{\pe}(u) = \vec{0}$
otherwise, and $\pe:\N \to \R^T$ is defined by a fixed randomly chosen injective map from $\N$
to the sphere $\set{x \in \R^T \mid \norm{x} = 1}$.
\end{definition}

We use the index function $\pe$ to encode the index at which an object instantiates the argument of a predicate or action schemas.
This usage of $\pe$ lets us address STRIPS-HGN's limitation of having a fixed maximum number of parameters for predicates and action schemas that is chosen before training.
Specifically, $\pe$ provides a numerically stable representation of an unbounded range of indexes that is agnostic to the maximum arity of the problem.
Moreover, $\pe$ improves generalisation to large and unseen indexes as they are mapped to normalised vectors already seen by the model.

IF draws inspiration from positional encoding functions used in Transformers~\cite{vaswani2017attention} and GNNs~\cite{li2020distance,dwivedi2021graph,wang2021equivariant} for encoding positions as vectors.
However, while Transformers and GNNs use positional encodings to encode all input tokens and graph nodes respectively, we use $\pe$ features only in the subset of nodes required to encode argument indexes in the lifted planning task.
Furthermore, positional encodings aim to correlate positions and their corresponding features, such that objects close to each other are given similar features.
In contrast, $\pe$ generates features for each index that are independent of one another.
Hence, $\pe$ features are randomly generated i.i.d.
for each index and also \emph{a priori} such that they can be used for domain-independent learning.

\section{Expressiveness}
We have defined three novel graph representations of planning tasks for the goal of learning domain-independent heuristics.
In this section we will categorise the expressiveness of such graph representations when used with MPNNs by identifying which domain-independent heuristics they are able to learn.
Our study also includes characterising the expressiveness of STRIPS-HGN~\cite{shen20stripshgn}, the previous work on learning domain-independent heuristics.
Fig.~\ref{fig:hierarchy} summarises the main theorems of this section via an expressiveness hierarchy.
Proofs of theorems are provided in a technical report~\cite{chen2023learning}.

We begin with a lower bound on what MPNNs can learn by showing that they can theoretically learn to imitate algorithms for computing $\hmax$ and $\hadd$ on our grounded graphs with the use of the approximation theorem for neural networks~\cite{cybenko1989approximation,hornik1989multilayer}.
We note that the theorem does not say anything about generalisability.

\begin{figure}
\centering
\tikzset{every picture/.style={scale=0.27, every picture/.style={}}} 
\newcommand{\hierarchyfontsize}{\scriptsize}
\newcommand{\hierarchyhfontsize}{\tiny}
\newcommand{\newgraphslinewidth}{0.4mm}
\newcommand{\oldcolour}{gray!50}
\begin{tikzpicture}

  \draw[rounded corners, draw=\oldcolour] (2, 1) rectangle (10, 3) {};
  \node at (10, 3) [font=\hierarchyhfontsize, inner sep=1pt, anchor=south east] {$h^{\max\!\!/\!\!\add}$};

  \draw[rounded corners, draw=\oldcolour] (2, 1) rectangle (12, 4) {};
  \node at (12, 4) [font=\hierarchyfontsize, inner sep=1pt, anchor=south east] {STRIPS-HGN};

  \draw[rounded corners, draw=\oldcolour] (2, 1) rectangle (18, 7) {};
  \node at (18, 7) [font=\hierarchyhfontsize, inner sep=1pt, anchor=south east] {$\hplus$};

  \draw[rounded corners, draw=\oldcolour] (0, 0) rectangle (20, 8) {};
  \node at (20, 8) [font=\hierarchyhfontsize, inner sep=1pt, anchor=south east] {$\hopt$};

  \draw[rounded corners, line width=\newgraphslinewidth] (0, 0) rectangle (8, 2) {};
  \node at (8, 2) [font=\hierarchyfontsize, inner sep=1pt, anchor=south east] {\llg{}};

  \draw[rounded corners, line width=\newgraphslinewidth] (1, 0) rectangle (14, 5) {};
  \node at (14, 5) [font=\hierarchyfontsize, inner sep=1pt, anchor=south east] {\slg};

  \draw[rounded corners, line width=\newgraphslinewidth] (1, 0) rectangle (16, 6) {};
  \node at (16, 6) [font=\hierarchyfontsize, inner sep=1pt, anchor=south east] {\flg};
\end{tikzpicture}
\caption{
Expressiveness hierarchy of MPNNs on graph representations with respect to STRIPS-HGN and the heuristics $h^{\max}$, $h^{\add}$, $h^+$ and $\hopt$.
Bold outlines represent new graphs.}\label{fig:hierarchy}
\end{figure}
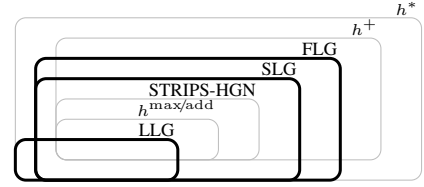

\begin{theorem}[MPNNs can learn $\hadd$ and $\hmax$ on grounded graphs]\label{thm:ground_h_add_max}
Let $L, B \in \N$, $\Graph \in \set{\slg, \flg}$, $\e > 0$ and $h \in \seta{\hadd, \hmax}$.
Then there exists a set of parameters $\params$ for an MPNN $\F_{\params}$ such that for all planning tasks $\Pi$, if naive dynamic programming for computing $h$ converges within $L$ iterations for $\Pi$, and $h(\Pi) \leq B$, then we have $\abs{h(\Pi) - \F_{\params}(\Graph(\Pi))} < \e$.
\end{theorem}

MPNNs acting on \slg{} and \flg{} are strictly more expressive than STRIPS-HGN.
The idea of the theorem is that STRIPS-HGN discards delete effects which prohibits it from learning $h^*$.
Furthermore, it is possible to imitate STRIPS-HGN with minor assumptions on MPNN architectures acting on either of our grounded graphs.

\begin{theorem}[MPNNs on grounded graphs are strictly more expressive than
STRIPS-HGN]\label{thm:strips_hgn} Let $\Graph \in \set{\slg, \flg}$.
Given any set of parameters $\Theta$ for a STRIPS-HGN model $\S_{\Theta}$, there is a set of parameters $\Phi$ for an MPNN $\F_{\Phi}$ such that for any pair of planning tasks $\Pi_1$ and $\Pi_2$ where $\S_{\Theta}(\Pi_1) \not= \S_{\Theta}(\Pi_2)$, we have $\F_{\Phi}(\Graph(\Pi_1)) \not= \F_{\Phi}(\Graph(\Pi_2))$.
Furthermore, there exists a pair of planning problems $\Pi_1$ and $\Pi_2$ such that there exists $\Phi$ where $\F_{\Phi}(\Graph(\Pi_1)) \not= \F_{\Phi}(\Graph(\Pi_2))$ but $\S_{\Theta}(\Pi_1) = \S_{\Theta}(\Pi_2)$ for all $\Theta$.
\end{theorem}

The first of our negative results is that MPNNs cannot learn $\hadd$ or $\hmax$ on the lifted \llg{} graph.
This is due to the graph being too condensed in the lifted version so that MPNNs cannot extract certain information for computing these heuristics.
The proof idea is to find a pair of planning tasks which appear symmetric to MPNNs in the \llg{} representation but have different $\hmax$ and $\hadd$ values.

\begin{theorem}[MPNNs cannot learn $\hadd$ and $\hmax$ on lifted
graphs]\label{thm:lifted_h_add_maxx} Let $h \in \set{\hadd, \hmax}$.
There exists a pair of planning tasks $\Pi_1$ and $\Pi_2$ with $h(\Pi_1) \not= h(\Pi_2)$ such that for any set of parameters $\Theta$ for an MPNN we have $\F_{\Theta}(\llg(\Pi_1)) = \F_{\Theta}(\llg(\Pi_2))$.
\end{theorem}

Next, we have that MPNNs cannot learn $\hplus$ and thus $\hopt$ on any of our graphs.
This result is not unexpected given that the expressiveness of MPNNs is bounded by the graph isomorphism class $\GI$ whose hardness is known to be in the low hierarchy of $\NP$, unlike $\hplus$ which is $\NP$-complete.
Similarly to the previous theorem, the proof follows the technique of finding a pair of planning tasks with different $\hplus$ values that are indistinguishable by MPNNs on any of our graphs.

\begin{theorem}[MPNNs cannot learn $\hplus$ or $\hopt$ with our graphs]\label{thm:h_plus_opt} Let $h \in \set{\hplus, \hopt}$ and $\Graph \in \set{\slg, \flg, \llg}$.
There exists a pair of planning tasks $\Pi_1$ and $\Pi_2$ with $h(\Pi_1) \not= h(\Pi_2)$ such that for any set of parameters $\Theta$ for an MPNN we have $\F_{\Theta}(\Graph(\Pi_1)) = \F_{\Theta}(\Graph(\Pi_2))$.
\end{theorem}

One may ask if it is possible to learn any approximation of $\hplus$ or $h^*$ on all planning problems.
Unfortunately, it is not possible to learn either absolute or relative approximations.
This is formalised in the following theorem, where the proof consists of a class of planning task pairs generalising the previous example.

\begin{theorem}[MPNNs cannot learn any approximation of $\hplus$ or $\hopt$]\label{thm:approx} Let $h \in \set{\hplus, \hopt}$, $\Graph \in \set{\slg, \flg, \llg}$ and $c > 0$.
There exists a pair of planning tasks $\Pi_1$ and $\Pi_2$ with $h(\Pi_1) \not= h(\Pi_2)$ such that for any set of parameters $\Theta$ for an MPNN we do not have $\abs{\F_{\Theta}(\Graph(\Pi_1)) - h(\Pi_1)}\leq c \wedge \abs{\F_{\Theta}(\Graph(\Pi_2)) - h(\Pi_2)}\leq c$.
Also, for any set of parameters we do not have $\abs{1 - {\F_{\Theta}(\Graph(\Pi_1))}/{h(\Pi_1)}}\leq c \wedge \abs{1 - {\F_{\Theta}(\Graph(\Pi_2))}/{h(\Pi_2)}}\leq c$.
\end{theorem}

We note that our theorems provide extreme upper and lower bounds on what MPNNs can learn with some of our graphs.
Although in the general case we have the negative result that we cannot learn $\hplus$ or $\hopt$, it is still possible to learn $\hopt$ on subclasses of planning tasks.
For example, \citet{staahlberg2022optimalpolicies} theoretically analyse which domains their MPNN architecture can learn $\hopt$ for, by using the well known result concerning the connection between MPNNs and 2-variable counting logics~\cite{cai1992optimal,barcelo2020logical}.
Furthermore, the results in this study concern MPNNs but there exist more expressive graph representation learning methods.
For example, under additional assumptions, the universal approximation theorem with random node initialisation by~\citet{abboud2021surprising} can be applied to learn $\hopt$.
Lastly, note that neither our results nor previous works examine the generalisability of learned heuristic functions.

\begin{table}
\renewcommand{\arraystretch}{0.75}
\centering
\scriptsize
\newcommand\macrofornuminstances[1]{\emph{#1}}
\setlength{\tabcolsep}{1.8pt}
\begin{tabularx}{\columnwidth}{c |l r |l r |l r| l} \toprule
domain & train & & validate & & test & & largest sol.
\\ \midrule
blocks & $b \in [3,10]$& \macrofornuminstances{40} & $b \in [11]$& \macrofornuminstances{3} & $b \in [15,100]$& \macrofornuminstances{90} & $b=75^{\ast}$ \\
ferry & $l,c\in[2,10]$& \macrofornuminstances{125}& $l,c\in[11]$& \macrofornuminstances{3} & $l,c \in [15,100]$ & \macrofornuminstances{90} & $l,c=100$ \\
gripper & $b\in[1,10]$& \macrofornuminstances{10} & $b \in[11]$& \macrofornuminstances{1} & $b \in [15,100]$ & \macrofornuminstances{18} & $b=100$ \\
n-puzzle & $n\in[2,4]$& \macrofornuminstances{100} & $n\in[5]$& \macrofornuminstances{3} & $n \in [5,9]$& \macrofornuminstances{50} & $n=6^{\ast}$ \\
sokoban & $n\in[5,7]$& \macrofornuminstances{60} & $n\in[8]$& \macrofornuminstances{3} & $n \in [8,12]$& \macrofornuminstances{90} & $n=12^{\ast}$ \\
spanner & $s,n\in[2,10]$& \macrofornuminstances{75} & $s,n \in[11]$& \macrofornuminstances{3} & $s,n\in[15,100]$& \macrofornuminstances{90} & $s,n=75$ \\
visitall & $n\in[3,10]$& \macrofornuminstances{24} & $n\in[11]$& \macrofornuminstances{3} & $n\in[15, 100]$& \macrofornuminstances{90} & $n=65$ \\
visitsome & $n\in[3,10]$& \macrofornuminstances{24} & $n\in[11]$& \macrofornuminstances{3} & $n\in[15, 100]$& \macrofornuminstances{90} & $n=95^*$ \\
\bottomrule\end{tabularx}
\caption{
Problem splits with sizes and number of tasks per domain.
Right most column indicates largest size problem solved with GOOSE.
Problem size is not completely correlated with difficult for domains marked $\ast$.
}\label{table:splits}
\end{table}

\section{Experiments}
We provide experiments in order to evaluate the effectiveness of our graph representations for use with both domain-dependent and domain-independent learning of heuristic functions, as well as to answer some open questions left behind in our theoretical discussion.
In order to do so, we introduce our GOOSE planner which combines graph generation, graph representation learning and domain-independent planning. Code is available at ~\cite{chen_2023_10404209}.

\subsubsection{GOOSE} The \textbf{G}raphs \textbf{O}ptimised f\textbf{O}r \textbf{S}earch \textbf{E}valuation (GOOSE) architecture represents planning tasks with one of the three graphs described previously (\slg{}, \flg{}, \llg{}) and uses MPNNs to learn heuristic functions for search.
During heuristic evaluation, GOOSE treats each state $s$ of a planning task $\gen{S,A,s_0,G}$ as a new planning subtask $\gen{S,A,s,G}$ which is then transformed into the chosen graph and fed into an MPNN.
We use a RGCN~\cite{schlichtkrull2018modeling} message passing step:
$
\h_u^{(t+1)} = \!
\sigma\Bigl(\mathbf{W}_{0}^{(t)}\h_u^{(t)} + 
\sum_{\iota \in \mathcal{R}} \bigoplus_{\edge{u,v}{\iota} \in \neighbours_{\iota}(u)} \mathbf{W}_{\iota}^{(t)} \h_v^{(t)}\Bigr),
$
where $\oplus$ denotes the aggregator over neighbours under different edge labels.
GOOSE uses the eager GBFS component of Fast Downward~\cite{helmert2006fast} for search, but calls the trained models for heuristic evaluation and parallelises the evaluation of successor states on GPUs for each opened node.
%
%
GOOSE only constructs graphs from static propositions computed by Fast Downward.

\newcommand{\scaleconst}{1}
\begin{figure}[t!]
\centering
\includegraphics[width=\scaleconst\columnwidth]{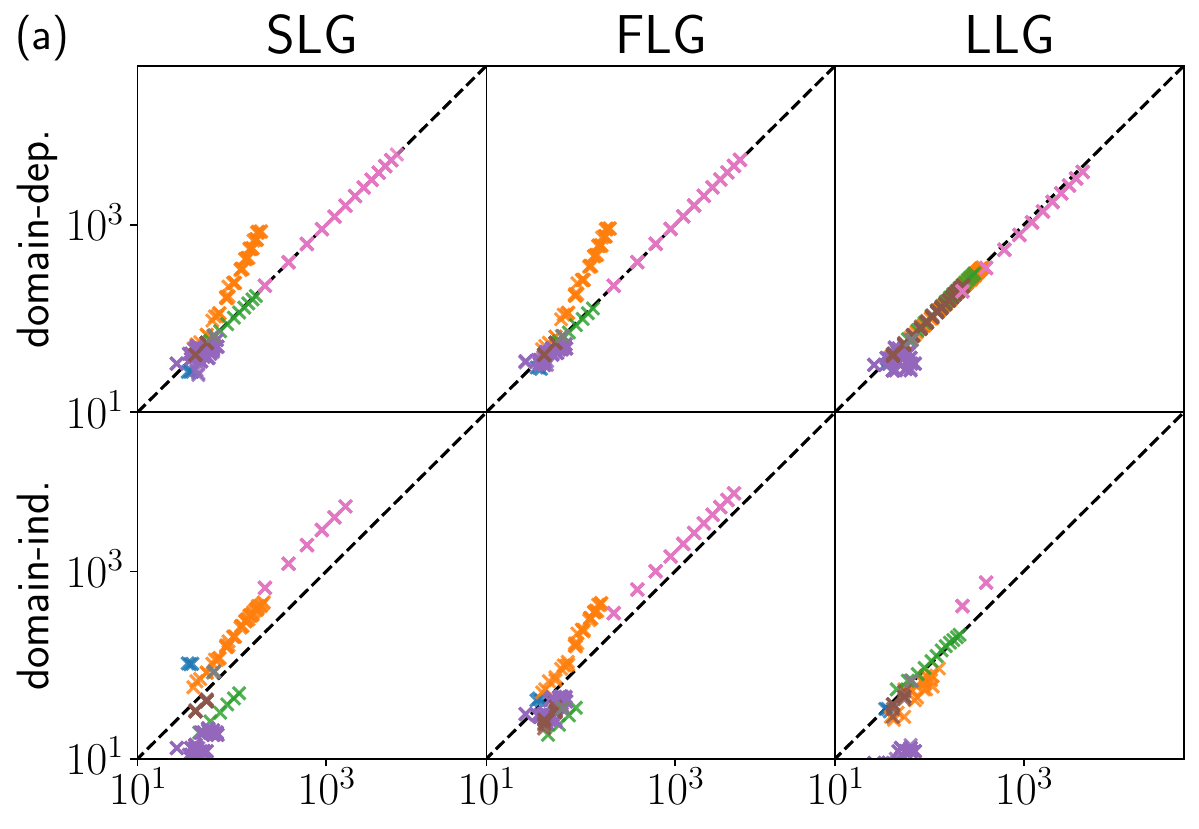}
\includegraphics[width=\scaleconst\columnwidth]{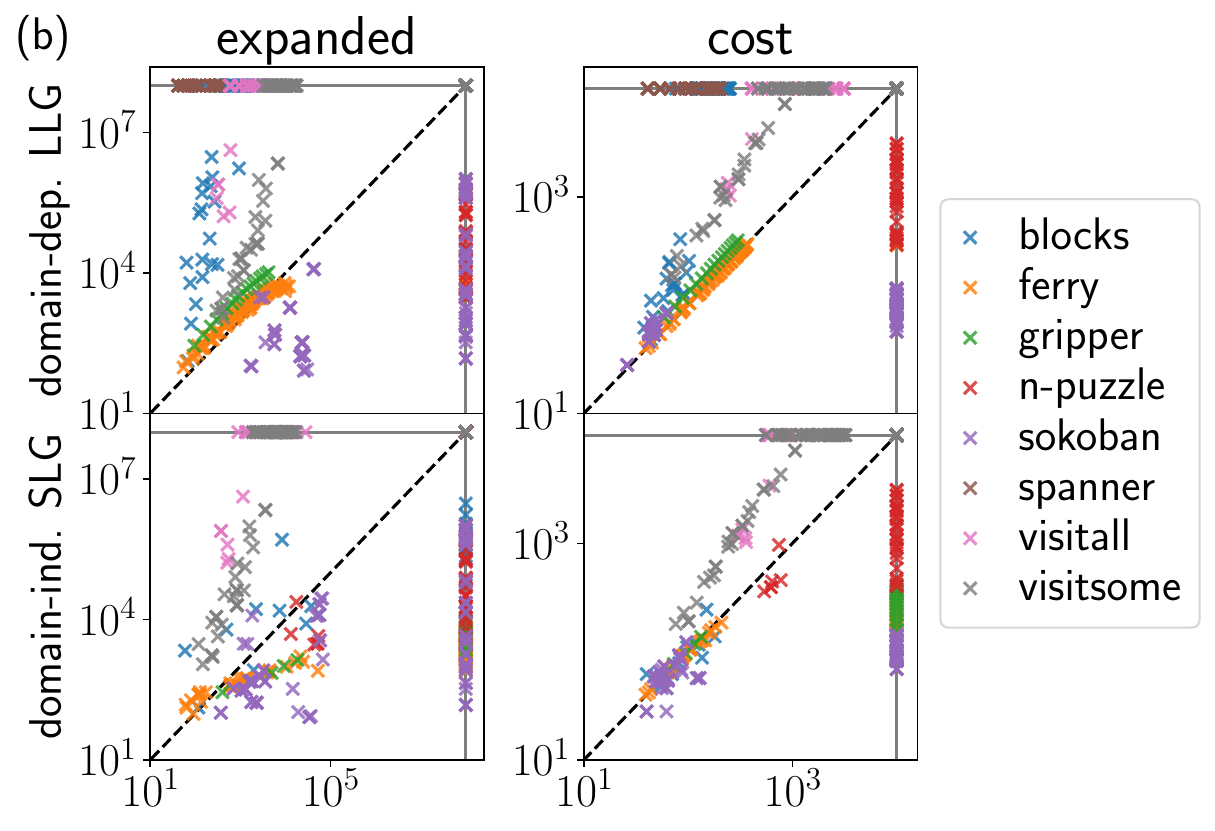}
\caption{
(a) GOOSE learned heuristics ($y$-axis) vs. $\hopt$ ($x$-axis).
No $n$-puzzle problem could have $h^*$ computed.
(b) $\hff$ ($y$-axis) vs.
GOOSE ($x$-axis) on number of expanded nodes (left) and plan cost (right).
Points on the bottom right triangles favour $\hff$ and on the top left triangles favour GOOSE.
Problems unsolved by a configuration get value set to the maximum of the plot's axis.
}
\label{fig:expanded_and_cost}
\label{fig:heuristic}
\end{figure}

\subsubsection{Setup} 
For domain-dependent heuristic learning, we train 5 models for each domain on optimal plans with problems specified in Tab.~\ref{table:splits}.
Each plan of length $h^*$ contributes states $s_0,s_1,\ldots,s_{g}$ with corresponding labels $h^*, h^*-1, \ldots,0$.
For domain-independent heuristic learning, we consider the problems and domains of the 1998 to 2018 IPC dataset, excluding the domains in Tab.~\ref{table:splits}. We train 5 models using optimal plans generated by scorpion \cite{seipp2020saturated} with a 30min cutoff time and unit costs.
%
%
In both settings, a model is trained with the Adam optimiser~\cite{kingma2014adam}, batch size 16, initial learning rate of 0.001 and MSE loss.
We schedule our learning rate by extracting 25\% of the training data and reducing the learning rate by a factor of 10 if the loss on this data subset did not decrease in the last 10 epochs.
Training is stopped when the learning rate becomes less than $10^{-5}$ on this subset, which often occurs within a few minutes.
Following a similar method to~\citet{ferber2020neural}, we select the best model for both settings by choosing the model which solves the most problems in the validation set (Tab.~\ref{table:splits}).
We break ties with the sum of number of expanded nodes, and the training loss.
%
%
We choose a hidden dimension of 64 with 8 message passing layers and the mean aggregator.
GOOSE is run with a single NVIDIA GeForce RTX 3090 GPU.

\subsubsection{Baselines} We evaluate against blind search, Fast Downward's implementation of eager GBFS with $\hff$, and domain-dependent STRIPS-HGN.
STRIPS-HGN was trained using the parameters described in the original paper but with the same dataset as GOOSE and is called from Fast Downward's eager GBFS for heuristic evaluation.
The other baselines are run on CPUs only.
All GOOSE configurations and baselines are run with a 600 second timeout and 8GB main memory.
We note that we also evaluated Powerlifted's~\cite{correa2020lifted} implementation of lifted eager GBFS with both $\hff$ and its extension with GOOSE but they do not offer any advantage over Fast Downward on the chosen problems.

\subsection{Results}
We recall from our theoretical study that it is not possible to learn any approximation of $\hopt$ over all planning tasks, although it may still be possible for certain planning subclasses or domains.
Furthermore, we can learn domain-dependent $\hopt$ heuristics for certain domains similarly to~\cite{staahlberg2022optimalpolicies}.
Thus, we propose and answer the following question empirically.

\subsubsection{How close to the optimal heuristic are our learned heurstics?}
To answer this question, we report the heuristic of the initial state $h(s_0)$ computed by our models on all planning problems for which we can compute the optimal heuristic either by a hard coded solver or an optimal planner in Fig.~\ref{fig:heuristic}(a).
In the domain-dependent training setting, GOOSE with \llg{} provides the best predictions over most domains except for Sokoban.
GOOSE achieves close to perfect heuristic estimates even as problems scale in size.
This can be explained by \llg{} encoding predicate information which the other graph representations do not have access to and for some domains, it suffices to count the predicates of true propositions to compute $\hopt$.
Meanwhile, MPNNs are not expressive enough to decode predicate information from the grounded graphs.
In the domain-independent training setting, GOOSE heuristics tend to overestimate $\hopt$ on VisitAll, whose perfect heuristic can often be computed by counting unreached goals, and underestimate on Sokoban.
The grounded graphs (\slg{} and \flg{}) underestimate on Gripper, Spanner and Sokoban.
This suggests GOOSE heuristics are overfitting on the training set for the more expressive grounded graphs as it is not possible to compute $\hopt$.
In particular, \flg{} is more prone to overfitting since it encodes additional planning task structure in the conversion of STRIPS to FDR.

\subsubsection{How useful are learned domain-dependent heuristics for search?}
To answer this question, we refer to Fig.~\ref{table:coverage} for a coverage table over all domains with various planners.
%
%
%
%
We notice that GOOSE with \llg{} trained in a domain-dependent fashion provides the best coverage on Blocksworld and Spanner, and is tied with Fast Downward's eager GBFS with \hff{} on Gripper.
GOOSE with \slg{} performs best on the grid based path finding domains VisitAll and VisitSome.
Meanwhile, \hff{} performs best on the remaining Ferry, $n$-puzzle and Sokoban domains.
%
%
However, all GOOSE configurations perform worse than blind search on Sokoban.
Even though it expands fewer nodes than blind search, the runtime cost of computing GOOSE heuristics is too high.
This may be due to the difficulty of the domain (\PSPACE-complete) as size increases, given that in problems with similar size to the training set GOOSE outperforms the other baselines.
STRIPS-HGN solves significantly fewer problems due to its slower evaluation on CPUs.
%

Fig.~\ref{fig:expanded_and_cost}(b) shows the number of node expansions and returned plan quality of the best performing domain-dependent GOOSE graph, \llg{}, against \hff{}.
In domains where one planner solves significantly more problems than the other, it also has fewer node expansions.
We also note that GOOSE generally has higher plan quality than
\hff{} over all problems which both planners were able to solve.

\renewcommand{\arraystretch}{0.75}
\begin{figure}[t!]
\newcommand{\domainwidth}{1.4cm}
\newcommand{\coveragewidth}{0.1775cm}
\scriptsize
\centering
\renewcommand{\header}[1]{#1}
\setlength{\tabcolsep}{2.5pt}
\newcommand{\asdf}{0.575cm}

\begin{tabularx}{\columnwidth}{l >{\raggedleft\arraybackslash}p{\asdf}>{\raggedleft\arraybackslash}p{\asdf}>{\raggedleft\arraybackslash}p{\asdf}>{\raggedleft\arraybackslash}p{\asdf}>{\raggedleft\arraybackslash}p{\asdf}>{\raggedleft\arraybackslash}p{\asdf}>{\raggedleft\arraybackslash}p{\asdf}>{\raggedleft\arraybackslash}p{\asdf}>{\raggedleft\arraybackslash}p{\asdf}>{\raggedleft\arraybackslash}p{\asdf}} \toprule
& \multicolumn{3}{c}{baselines} & \multicolumn{3}{c}{domain-dep.} & \multicolumn{3}{c}{domain-ind.} \\\cmidrule(lr){2-4} \cmidrule(lr){5-7} \cmidrule(lr){8-10} 
& \header{\blind} & \header{\hfftable} & \header{\shgn} & \header{\slg} & \header{\flg} & \header{\llg} & \header{\slg} & \header{\flg} & \header{\llg} \\ \midrule
blocks (90)&\zerocell{0}&\second{19}&\zerocell{0}&\zerocell{0}&\normalcell{6}&\first{62}&\third{9}&\normalcell{8}&\normalcell{6}\\
ferry (90)&\zerocell{0}&\first{90}&\zerocell{0}&\normalcell{32}&\third{33}&\second{88}&\normalcell{28}&\normalcell{22}&\normalcell{2}\\
gripper (18)&\normalcell{1}&\first{18}&\normalcell{5}&\third{9}&\normalcell{6}&\first{18}&\normalcell{5}&\normalcell{3}&\third{9}\\
n-puzzle (50)&\zerocell{0}&\first{36}&\zerocell{0}&\second{10}&\second{10}&\zerocell{0}&\normalcell{6}&\normalcell{3}&\zerocell{0}\\
sokoban (90)&\second{74}&\first{90}&\normalcell{10}&\normalcell{31}&\normalcell{29}&\normalcell{34}&\third{45}&\normalcell{40}&\normalcell{15}\\
spanner (90)&\zerocell{0}&\zerocell{0}&\zerocell{0}&\zerocell{0}&\zerocell{0}&\first{60}&\zerocell{0}&\zerocell{0}&\zerocell{0}\\
visitall (90)&\zerocell{0}&\normalcell{6}&\normalcell{25}&\second{46}&\first{50}&\third{44}&\normalcell{16}&\normalcell{41}&\zerocell{0}\\
visitsome (90)&\normalcell{3}&\normalcell{26}&\normalcell{33}&\second{72}&\normalcell{39}&\third{65}&\first{73}&\third{65}&\normalcell{15}\\
\bottomrule\end{tabularx}

\caption{
Coverage of planners and GOOSE over various domains.
%
%
Cell intensities indicate the top 3 planners per row.
}\label{table:coverage}
\end{figure}

\begin{figure}[t!]
\renewcommand{\header}[1]{#1}
\centering\scriptsize
\begin{tabularx}{\columnwidth}{l l Y Y Y Y Y Y} \toprule & & \multicolumn{3}{c}{domain-dep.} & \multicolumn{3}{c}{domain-ind.} \\ 
\cmidrule(lr){3-5} 
\cmidrule(lr){6-8} 
aggr. & L & \header{\slg} & \header{\flg} & \header{\llg} & \header{\slg} & \header{\flg} & \header{\llg} \\ \midrule
\multirow{4}{*}{mean}&4&\cellcolor{\colorofcell!11.955168119551683}{0.40}&\cellcolor{\colorofcell!12.844333748443338}{0.43}&\cellcolor{\colorofcell!28.281444582814444}{0.94}&\cellcolor{\colorofcell!5.753424657534247}{0.19}&\cellcolor{\colorofcell!4.632627646326276}{0.15}&\cellcolor{\colorofcell!5.491905354919055}{0.18}\\
&8&\cellcolor{\colorofcell!15.840597758405977}{\textbf{0.53}}&\cellcolor{\colorofcell!12.141967621419676}{0.40}&\cellcolor{\colorofcell!30.0}{\textbf{1.00}}&\cellcolor{\colorofcell!11.462017434620176}{\textbf{0.38}}&\cellcolor{\colorofcell!9.579078455790786}{0.32}&\cellcolor{\colorofcell!9.825653798256539}{0.33}\\
&12&\cellcolor{\colorofcell!13.188044831880449}{0.44}&\cellcolor{\colorofcell!11.140722291407224}{0.37}&\cellcolor{\colorofcell!25.516811955168123}{0.85}&\cellcolor{\colorofcell!11.170610211706103}{0.37}&\cellcolor{\colorofcell!9.466998754669989}{0.32}&\cellcolor{\colorofcell!6.313823163138232}{0.21}\\
&16&\cellcolor{\colorofcell!9.302615193026153}{0.31}&\cellcolor{\colorofcell!5.334993773349938}{0.18}&\cellcolor{\colorofcell!22.528019925280198}{0.75}&\cellcolor{\colorofcell!10.916562889165629}{0.36}&\cellcolor{\colorofcell!9.676214196762142}{\textbf{0.32}}&\cellcolor{\colorofcell!3.4744707347447075}{0.12}\\
\multirow{4}{*}{max}&4&\cellcolor{\colorofcell!13.748443337484433}{0.46}&\cellcolor{\colorofcell!14.981320049813203}{\textbf{0.50}}&\cellcolor{\colorofcell!26.74968866749689}{0.89}&\cellcolor{\colorofcell!9.840597758405977}{0.33}&\cellcolor{\colorofcell!8.60772104607721}{0.29}&\cellcolor{\colorofcell!8.891656288916563}{0.30}\\
&8&\cellcolor{\colorofcell!12.36612702366127}{0.41}&\cellcolor{\colorofcell!12.8293897882939}{0.43}&\cellcolor{\colorofcell!26.450809464508097}{0.88}&\cellcolor{\colorofcell!10.737235367372355}{0.36}&\cellcolor{\colorofcell!9.041095890410958}{0.30}&\cellcolor{\colorofcell!15.466998754669985}{\textbf{0.52}}\\
&12&\cellcolor{\colorofcell!10.841843088418432}{0.36}&\cellcolor{\colorofcell!12.8293897882939}{0.43}&\cellcolor{\colorofcell!23.98505603985056}{0.80}&\cellcolor{\colorofcell!3.698630136986302}{0.12}&\cellcolor{\colorofcell!7.202988792029888}{0.24}&\cellcolor{\colorofcell!11.61892901618929}{0.39}\\
&16&\cellcolor{\colorofcell!12.186799501867995}{0.41}&\cellcolor{\colorofcell!10.699875466998755}{0.36}&\cellcolor{\colorofcell!15.952677459526775}{0.53}&\cellcolor{\colorofcell!1.7185554171855542}{0.06}&\cellcolor{\colorofcell!7.1955168119551685}{0.24}&\cellcolor{\colorofcell!5.940224159402242}{0.20}\\

\bottomrule\end{tabularx}

\caption{
Total coverage normalised per domain of GOOSE over various parameters and training paradigms, and normalised again by the coverage of the best performing configuration.
Higher scores are better and the maximum score is 1.
The best scores per \emph{column} are highlighted in bold.
}\label{table:params}
\end{figure}

\subsubsection{How useful are learned domain-independent heuristics for search?}
We again refer to Fig.~\ref{table:coverage} for the coverage of GOOSE trained with domain-independent heuristics.
With the exception of Sokoban, domain-independent GOOSE outperforms blind search which suggests that the learned domain-independent heuristics have some informativeness.
This is supported by Fig.~\ref{fig:heuristic}(a) which shows that in most domains domain-independent heuristics provide a mostly-linear approximation of $\hopt$.
Most notably, domain-independent grounded graphs still outperform $\hff$ on VisitAll and VisitSome, and domain-independent \llg{} is able to solve some Spanner problems.

The best performing domain-independent GOOSE graph with 8 message passing layers and mean aggregator is the grounded graph \slg{}.
It provides enough information to learn domain-independent heuristics with MPNNs in comparison to \llg{}, but also does not provide too much information to prevent overfitting in comparison to \flg{} which computes additional structure.
Domain-independent GOOSE with \slg{} returns better quality plans, and expands fewer nodes than \hff{} on VisitAll, VisitSome, and more than half the Blocksworld instances which both planners were able to solve.
Domain-independent GOOSE also outperforms or ties with \emph{domain-dependent} STRIPS-HGN across all domains except VisitAll.
However, domain-independent GOOSE generally expands more nodes and returns lower quality plans than their domain-dependent trained variants with the same graph.

\subsubsection{How important is finding the right graph neural network parameters?}
We report the normalised coverage of GOOSE with hyperparameters $L\in\{4,8,12,16\}$ layers and $\oplus\in\{\max,\text{mean}\}$ aggregator in Fig.~\ref{table:params}.
%
We omitted results with the sum aggregator as it yielded unstable training and poor predictions.
Increasing the number of layers theoretically improves informativeness and accuracy of predictions but requires longer evaluation time and is more difficult to train.
There is no single set of parameters that performs well over all graphs and training settings.
Generally 4 or 8 layers result in similar coverages for domain-dependent training, and 8 or 12 layers for domain-independent training, while increasing the number of layers beyond this results in worse performance due to the aforementioned reasons.
We note that the effectiveness of max and mean aggregations vary with the graph representation and domain as both aggregators lose information in different ways.
However, in the domain-dependent setting, \llg{} with the mean aggregator generally outperforms the max aggregator given that the model can recover the information lost during normalisation through the grouping of edge labels and node types.

\subsubsection{How long do GOOSE evaluations take?}
GOOSE on a single core CPU takes 0.2-0.9s to perform a full GNN evaluation on grounded graphs, and 0.2-0.3s on lifted graphs.
With \emph{optimal} GPU usage, grounded graphs take 0.1-5ms per state evaluation and lifted graphs take 0.07-0.7ms.
Optimal usage is achieved when the batch size is greater than 32 for lifted graphs and 4 for grounded graphs.
In our experiments, the grounded graphs were able to optimally use the GPU while the lifted graphs were not, resulting in a higher average state evaluation time.
We note that evaluating states of successor nodes further in the queue to increase the batch size is also not optimal as the expanded states may never be evaluated in sequential GBFS.
Evaluation of heuristics on GPUs is almost always faster than on CPU due to the parallel execution of GNN matrix and scatter operations.

\section{Conclusion}
We have constructed various novel graph representations of planning problems for the task of learning domain-independent heuristics.
In particular we provide the first domain-independent graph representation of lifted planning.
All our new models are also complemented by a theoretical analysis of their expressive power in relation to domain-independent heuristics and the previous work on learning domain-independent heuristics, STRIPS-HGN.
We also construct the GOOSE planner using heuristic search with heuristics learned from our new graph representations.
GOOSE has also been optimised for runtime with the use of GPU batch evaluation and is able to solve significantly larger problems than those seen in the training set, vastly surpassing STRIPS-HGN learned heuristics, and outperforming the $\hff$ heuristic on several domains.
It remains for future work to implement search algorithms used by stronger satisficing planners in GOOSE, and to optimise GPU utilisation when computing heuristics.
Furthermore GOOSE can be extended to predict deadends alongside a heuristic for further pruning the search space.
Lastly, we aim to improve the expressiveness of learned heuristics by leveraging stronger graph representation learning techniques.

\section*{Acknowledgements}
The authors would like to thank the reviewers and Rostislav Hor\v{c}\'{i}k for their comments. 
The computational resources for this project were partially provided by the Australian Government through the National Computational Infrastructure (NCI) under the ANU Startup Scheme. 
Sylvie Thi{\'{e}}baux's work was supported by Australian Research Council grant DP220103815, by the Artificial and Natural Intelligence Toulouse Institute (ANITI), and by the European Union's Horizon Europe Research and Innovation program under the grant agreement TUPLES No. 101070149. 

\bibliography{bib.bib}

\appendix
\section{Proofs of Theorems}

This appendix includes the proofs of all theorems. For ease of reference, theorem statements are also repeated here.

\begin{algorithm}
  \caption{Naive dynamic programming for computing $\hadd$ and $\hmax$}\label{alg:naivedp}
  \KwData{Propositional STRIPS planning task $\Pi=\probstrips$, desired heuristic $h \in \set{\hadd, \hmax}$}
  \KwResult{$h(s) \in \N$}
  \lIf{$h=\hadd$}{$\oplus \la \sum$}
  \lElseIf{$h=\max$}{$\oplus \la\max$}
  $h^{(0)}[p] \la 0, \quad\forall p \in s_0$ \\
  $h^{(0)}[p] \la \infty, \quad\forall p \in P\setminus s_0$ \\
  \For{$i=1,\ldots$}{
    \For{$a \in A$}{
      $h^{(i)}[a] \la \oplus_{p \in \pre(a)} h^{(i-1)}[p]$ \label{alg:line:ha}\\
    }
    \For{$p \in P$}{
      $h^{(i)}[p] \la \min\biglr{h^{(i-1)}[p], \min_{a \in A, p \in \add(a)} h^{(i)}[a] + c(a)}$ \label{alg:line:hp}\\
    }
    \If{$h^{(i)} = h^{(i-1)}$}{
      \Return{$\oplus_{p \in g} h^{(i)}[p]$} \\
    }
  }
\end{algorithm}

\begin{theorem}[MPNNs can learn $\hadd$ and $\hmax$ on grounded graphs]\label{thm:ground_h_add_max}
  Let $L, B \in \N$, $\Graph \in \set{\slg, \flg}$, $\e > 0$ and $h \in \seta{\hadd, \hmax}$. Then there exists a set of parameters $\params$ for an MPNN $\F_{\params}$ such that for all planning tasks $\Pi$, if naive dynamic programming for computing $h$ (Alg.~\ref{alg:naivedp}) converges within $L$ iterations for $\Pi$, and $h(s_0) \leq B$, then we have $\abs{h(s_0) - \F_{\params}(\Graph(\Pi))} < \e$.
\end{theorem}
\begin{proof}
  The main idea of the proof is that we can encode Alg.~\ref{alg:naivedp} for computing $h$ into an MPNN using a correct choice of continuous bounded functions and aggregation operators and using the approximation theorem to find parameters in order to achieve the desired function. We will assume unitary cost actions and note that the below proof can be generalised to account for general cost actions. We first deal with the case where $h=\hmax$ and $\mathcal{G} = \dlg{}$, where \dlg{} is the graph \slg{} but without delete edges. The proof generalises to \slg{} as an MPNN can learn to ignore delete edges.

  Let $x^{(u)} \in \R^3$ be the feature of node $u$. By definition of \dlg{} as the graph \slg{} (Def.~3.1) with no delete edges, it is defined by $x^{(u)}_0=1$ if $u$ corresponds to a proposition node, else $x^{(u)}_0=0$ when $u$ corresponds to an action node $a$. Furthermore, $x^{(u)}_1=1$ if $u$ is a proposition in the initial state and $x^{(u)}_2=1$ if $u$ is a goal proposition. Note that it is possible that $x^{(u)}_1=x^{(u)}_2=1$ when a proposition is both a goal condition and in the initial state. If not mentioned, we have that $x^{(u)}_i=0$ everywhere else. 
  
  Then we will construct a MPNN with $2L+2$ layers. For the first layer we have an embedding layer which ignores neighbourhood nodes with $\aggr^{(0)} = \vec{0}$ and $\phi^{(0)}(\h_u, \h_N) = f_{\text{emb}}(\h_u)$. Let $K$ be the finite set of possible node features in a \dlg{} representation of a planning task. Then $f_{\text{emb}}:K \to \R^3$ is defined by 

  \begin{align}
    f_{\text{emb}}([0,0,0]^{\top}) &= [0,0,0]^{\top} \label{eq:emb1}\\
    f_{\text{emb}}([1,0,0]^{\top}) &= [B,0,1]^{\top} \\
    f_{\text{emb}}([1,0,1]^{\top}) &= [B,1,1]^{\top} \\
    f_{\text{emb}}([1,1,0]^{\top}) &= [0,0,1]^{\top} \\
    f_{\text{emb}}([1,1,1]^{\top}) &= [0,1,1]^{\top}.\label{eq:emb5}
  \end{align}

  This first round of message passing updates corresponds to the initialisation step of the heuristic algorithm with $B$ representing infinity values. We also note that after applying $\aggr^{(0)}$ and $\phi^{(0)}$ and throughout the remaining forward pass of the MPNN, node embeddings will have the form $[x_0,x_1,x_2]$ which encode information about their corresponding proposition or action during the execution of the $\hmax$ algorithm where
  \begin{itemize}
    \item $x_0$ corresponds to the intermediate $h$ values computed in the $\hmax$ algorithm, 
    \item $x_1$ signifies whether the node corresponds to a goal node, and 
    \item $x_2$ determines if the node is a proposition or action node.
  \end{itemize}
  
  The next $2L$ layers use the component wise max aggregation function $\aggr=\max$ and alternates between setting $\phi^{(l)} (\h_u, \h_N) = f_a([\h_u \concat \h_N])$ and $\phi^{(l+1)}(\h_u, \h_N) = f_p([\h_u \concat \h_N])$ where $f_a:\R^6 \to \R^3$ and $f_p:\R^6 \to \R^3$ are defined by 
  \begin{align}
    f_a\lr{\begin{bmatrix}
      x_0\\ x_1\\ x_2\\ y_0 \\y_1 \\y_2
    \end{bmatrix}}
    &
    =\begin{bmatrix}
      x_0x_2 - (1-x_2)y_0 \\
      x_1x_2 \\
      x_2^2
    \end{bmatrix}, 
    \\
    f_p\lr{\begin{bmatrix}
      x_0\\ x_1\\ x_2\\ y_0 \\y_1 \\y_2
    \end{bmatrix}}
    &
    =\begin{bmatrix}
      \min(x_0, -y_0+1)x_2 \\
      x_1x_2 \\
      x_2^2
    \end{bmatrix}. \label{eq:encoding}
  \end{align}

  These functions correspond to the iterative updates of $h^{(l)}[a]$ and $h^{(l)}[p]$ in Alg.~\ref{alg:naivedp}, recalling that $L$ is the number of iterations it takes for the algorithm converges. More specifically, suppose we have a node $u$ with embedding $\h_u=[x_0,x_1,x_2]$ and aggregated embedding from its neighbours $\h_N=[y_0,y_1,y_2]$. Then we have two cases.
  \begin{itemize}
    \item If $x_2 = 0$, indicating that the node $u$ corresponds to a n action, then we get 
    \begin{align}
      f_a([\h_u \concat \h_N]) &= [-y_0, 0, 0] \label{eq:f_a_for_action}\\
      f_p([\h_u \concat \h_N]) &= [0, 0, 0]. \label{eq:f_p_for_action}
    \end{align}
    Eq.~\ref{eq:f_a_for_action} corresponds to Line~\ref{alg:line:ha} in Alg.~\ref{alg:naivedp} where $-y_0$ contains the negative of $h[a]$. We take the negative since we are restricted to using $\max$ aggregators only\footnote{As $\min$ aggregators conflict with ReLU activation functions commonly seen in neural networks.} which in turn means we require taking maximums of negatives in order to mimic the minimum aggregator later in Line~\ref{alg:line:hp} of the same algorithm. Eq.~\ref{eq:f_p_for_action} corresponds to Line~\ref{alg:line:hp} but since this line only affects propositions and $h[a]$ values do not need to be stored after execution of this line, we set $\h_u$ to zero.
    \item If $x_2 = 1$, indicating that the node $u$ corresponds to a proposition, then we get
    \begin{align}
      f_a([\h_u \concat \h_N]) &= [x_0, x_1, x_2] \label{eq:f_a_for_prop}\\
      f_p([\h_u \concat \h_N]) &= [\min(x_0, -y_0+1), x_1, x_2]. \label{eq:f_p_for_prop}
    \end{align}
    We recall $f_a$ corresponds to Line~\ref{alg:line:ha} which only affects $h[a]$ values. Given that we require storing $h[p]$ values throughout the whole algorithm, $f_a$ acts as the identity function on $\h_u$ for proposition nodes as seen in Eq.~\ref{eq:f_a_for_prop}. This is in contrast to $f_p$ which acts as the zero function on $\h_u$ for action nodes. Eq.~\ref{eq:f_p_for_prop} corresponds to Line~\ref{alg:line:hp} where $-y_0$ is equivalent to the $\min_{a \in A, p \in \add(a)} h[a] = \max_{a \in A, p \in \add(a)} -h[a]$ term by definition of \dlg{}, $\aggr$ and $f_a$ acting on action node embeddings.
  \end{itemize} 
  
  We append a final layer to the network where we ignore neighbourhood nodes with $\aggr^{(2L+1)} = \vec{0}$ and $\phi^{(2L+1)}([x_0,x_1,x_2]^{\top}, \h_N) = x_0x_1$. In combination with a max readout function $\Phi$, this corresponds to computing the final heuristic value. The above encoding of Alg.~\ref{alg:naivedp} has also been experimentally verified to be correct.

  In order to satisfy the neural network component of the MPNN, we replace the $\phi^{(i)}$ for $i=0,\ldots,2L+1$ with feedforward networks. Noting that we have finitely many layers we can choose small enough fractions of $\e$ for the universal approximation theorem for neural networks~\citep{hornik1989multilayer,cybenko1989approximation} to approximate the continuous functions $\phi^{(i)}$ whose domain is bounded in the ball of radius $B$ in order to achieve our result.

  The encoding for $\hsum$ is the same except we use a sum aggregator $\aggr = \sum$ and readout.

  For the case of the other \flg{}, we note that the $\hmax$ and $\hsum$ algorithm for FDR problems and hence \flg{} graph representations work in the obvious way by compiling FDR planning tasks into propositional STRIPS planning task by treating variable-value pairs in FDR problems as propositional facts.
\end{proof}

\begin{theorem}[MPNNs on grounded graphs are strictly more expressive than STRIPS-HGN]\label{thm:strips_hgn}
  Let $\Graph \in \set{\slg, \flg}$. Given any set of parameters $\Theta$ for a STRIPS-HGN model $\S_{\Theta}$, there is a set of parameters $\Phi$ for an MPNN $\F_{\Phi}$ such that for any pair of planning tasks $\Pi_1$ and $\Pi_2$ where $\S_{\Theta}(\Pi_1) \not= \S_{\Theta}(\Pi_2)$, we have $\F_{\Phi}(\Graph(\Pi_1)) \not= \F_{\Phi}(\Graph(\Pi_2))$. Furthermore, there exists a pair of planning problems $\Pi_1$ and $\Pi_2$ such that there exists $\Phi$ where $\F_{\Phi}(\Graph(\Pi_1)) \not= \F_{\Phi}(\Graph(\Pi_2))$ but $\S_{\Theta}(\Pi_1) = \S_{\Theta}(\Pi_2)$ for all $\Theta$.
\end{theorem}
\begin{proof}[Proof sketch]
  To show the first part of the theorem, we describe how to construct an MPNN $\F_{\Phi}$ acting on \slg{} and \flg{} corresponding to a given STRIPS-HGN~\cite{shen20stripshgn} model $\S_{\Theta}$ such that for any pair of planning tasks $\Pi_1$ and $\Pi_2$ where $\S_{\Theta}(\Pi_1) \not= \S_{\Theta}(\Pi_2)$, we have $\F_{\Phi}(\Graph(\Pi_1)) \not= \F_{\Phi}(\Graph(\Pi_2))$. First we note that each STRIPS-HGN hypergraph message passing layer can be emulated by two MPNN message passing layers. We note that the STRIPS-HGN aggregation function is not permutation invariant as it requires ordering the messages it receives before concatenating them and updating the aggregated feature. This can similarly be done for a MPNN. Another difference with STRIPS-HGN and MPNNs is the usage of global features that are updated with each message passing layers. The MPNN framework can also be extended to make use of global features, for example by appending a virtual node to the whole input graph, and using different weights for the message passing functions associated with the virtual node. Lastly, STRIPS-HGN uses the same weights for each message passing layer and this may also be done for an MPNN.

  For the second part of the theorem, we note that for any planning problem $\Pi$ by definition of STRIPS-HGN, $\S_{\Theta}(\Pi) = \S_{\Theta}(\Pi^+)$ for all parameters $\Theta$ where $\Pi^+$ is the delete relaxation of $\Pi$. Now consider the STRIPS problem $\Pi = \gen{P, A, s_0, G}$ with $P = G = \set{p_0, p_1}$, $s_0 = \set{p_0}$, and $A = \set{a_0, a_1}$ where both $a_0$ and $a_1$ have empty precondition and 
  \begin{align*}
    \add(a_0) &= \set{p_1}, &\del(a_0) &= \set{a_0} \\
    \add(a_1) &= \set{a_1}, &\del(a_1) &= \emptyset.
  \end{align*}
  Then the optimal plan for $\Pi$ has cost 2, while the optimal plan cost of its delete relaxation $\Pi^+$ is 1. There exists a set of parameters for an MPNN $\F_{\Phi}$ acting on $\Graph \in \set{\slg, \flg}$ such that $\F_{\Phi}(\Graph(\Pi)) = 2 \not= 1 = \F_{\Phi}(\Graph(\Pi^+))$.
\end{proof}

\begin{theorem}[MPNNs cannot learn $\hadd$ and $\hmax$ on lifted graphs]\label{thm:lifted_h_add_max}
  Let $h \in \set{\hadd, \hmax}$. There exists a pair of planning tasks $\Pi_1$ and $\Pi_2$ with $h(\Pi_1) \not= h(\Pi_2)$ such that for any set of parameters $\Theta$ for an MPNN we have $\F_{\Theta}(\llg(\Pi_1)) = \F_{\Theta}(\llg(\Pi_2))$.
\end{theorem}
\input{figures/hadd_hmax.tex}
\begin{proof}
  Consider the two (delete free) lifted problems $P_1 = \langle\predicates, \objects, \actionschema, s_0^{(1)}, G\rangle$ and $P_2 = \langle\predicates, \objects, \actionschema, s_0^{(2)}, G\rangle$ with $\predicates = \set{Q(x_1, x_2), W(x_1, x_2)}$, $\objects = \set{o_1,o_2}$, $s_0^{(1)} = \set{Q(o_1,o_2), Q(o_2,o_1)}$, $s_0^{(2)} = \set{Q(o_1,o_1), Q(o_2,o_2)}$, $G = \set{{W(o_1,o_2), W(o_2,o_1)}}$ and one action schema $\actionschema = \set{a}$ with $\Delta(a) = \set{\d_1, \d_2}$, $\pre(a) = \set{Q(\d_1,\d_2)}$, $\add(a) = \set{W(\d_1,\d_2)}$ and $\del(a) = \emptyset$. 
  
  By definition $P_1$ can be solved with a plan consisting of $a(o_1, o_2)$ and $a(o_2, o_1)$ in either order and the corresponding heuristic values are $\hmax(s_0^{(1)})=\hsum(s_0^{(1)})=1$. On the other hand $P_2$ is unsolvable in which case we have $\hmax(s_0^{(2)})=\hsum(s_0^{(2)})=\infty$. 

  The graphs are indistinguishable by the WL algorithm where we colour nodes by mapping their features to the set of natural numbers for the $\llg{}$ graphs, given that the set of possible node features is countable. Note that it is possible to extend the WL algorithm to deal with edge labelled graphs by replacing each labelled edge with a coloured node connected to the edge's endpoints. Fig.~\ref{fig:lifted_counterexample} illustrates the graph representations for \llg{}. Then the result follows by the contrapositive of \cite[Lem.~2]{xu2018powerful} as WL assigns the same output for both graphs, and hence any MPNN also assigns the same output.
\end{proof}

\begin{theorem}[MPNNs cannot learn $\hplus$ or $\hopt$ with our graphs]\label{thm:h_plus_opt}
  Let $h \in \set{\hplus, \hopt}$ and $\Graph \in \set{\slg, \flg, \llg}$. There exists a pair of planning tasks $\Pi_1$ and $\Pi_2$ with $h(\Pi_1) \not= h(\Pi_2)$ such that for any set of parameters $\Theta$ for an MPNN we have $\F_{\Theta}(\Graph(\Pi_1)) = \F_{\Theta}(\Graph(\Pi_2))$.
\end{theorem}
\newcommand\sizeforhopt{0.75cm}
\newcommand\yshifta{1.6}
\newcommand\yshiftb{0.4}
\newcommand\xshift{1}
\newcommand\textsize{\scriptsize}
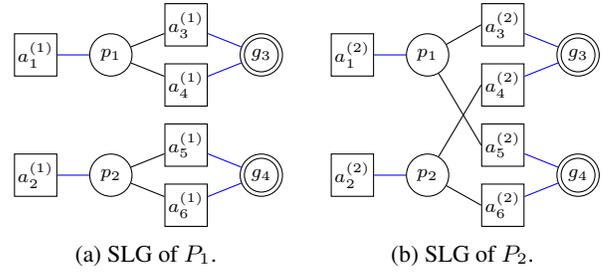
\begin{figure}[t]
  \centering
  \begin{subfigure}{0.49\columnwidth}
    \centering
    \begin{tikzpicture}[square/.style={regular polygon,regular polygon sides=4}]
      \node[shape=rectangle,draw=black,minimum size=0.75*\sizeforhopt] (A1) at (0,0) {};
      \node[shape=rectangle,draw=black,minimum size=0.75*\sizeforhopt] (A2) at ([shift={(0,-\yshifta)}]A1) {};
      \node[shape=circle,draw=black,minimum size=0.75*\sizeforhopt] (P1) at ([shift={(\xshift,0)}]A1) {};
      \node[shape=circle,draw=black,minimum size=0.75*\sizeforhopt] (P2) at ([shift={(\xshift,0)}]A2) {};
      \node[shape=rectangle,draw=black,minimum size=0.75*\sizeforhopt] (A3) at ([shift={(\xshift,+\yshiftb)}]P1) {};
      \node[shape=rectangle,draw=black,minimum size=0.75*\sizeforhopt] (A4) at ([shift={(\xshift,-\yshiftb)}]P1) {};
      \node[shape=rectangle,draw=black,minimum size=0.75*\sizeforhopt] (A5) at ([shift={(\xshift,+\yshiftb)}]P2) {};
      \node[shape=rectangle,draw=black,minimum size=0.75*\sizeforhopt] (A6) at ([shift={(\xshift,-\yshiftb)}]P2) {};
      \node[shape=circle,draw=black,minimum size=0.75*\sizeforhopt] (G3) at ([shift={(2*\xshift,0)}]P1) {};
      \node[shape=circle,draw=black,minimum size=0.75*\sizeforhopt] (G4) at ([shift={(2*\xshift,0)}]P2) {};
      \node[shape=circle,draw=black,minimum size=0.6*\sizeforhopt] (G33) at (G3) {};
      \node[shape=circle,draw=black,minimum size=0.6*\sizeforhopt] (G44) at (G4) {};

      \node at (A1) {\textsize$a_1^{(1)}$};
      \node at (A2) {\textsize$a_2^{(1)}$};
      \node at (P1) {\textsize$p_1$};
      \node at (P2) {\textsize$p_2$};
      \node at (A3) {\textsize$a_3^{(1)}$};
      \node at (A4) {\textsize$a_4^{(1)}$};
      \node at (A5) {\textsize$a_5^{(1)}$};
      \node at (A6) {\textsize$a_6^{(1)}$};
      \node at (G3) {\textsize$g_3$};
      \node at (G4) {\textsize$g_4$};

      \path [-,draw=blue] (A1) edge (P1);

      \path [-,draw=blue] (A2) edge (P2);

      \path [-] (P1) edge (A3);

      \path [-] (P1) edge (A4);
      
      \path [-] (P2) edge (A5);
      
      \path [-] (P2) edge (A6);
      
      \path [-,draw=blue] (A3) edge (G3);
      
      \path [-,draw=blue] (A4) edge (G3);

      \path [-,draw=blue] (A5) edge (G4);

      \path [-,draw=blue] (A6) edge (G4);
    \end{tikzpicture}
    \caption{\slg{} of $P_1$.}
  \end{subfigure}
  \begin{subfigure}{0.49\columnwidth}
    \centering
    \begin{tikzpicture}[square/.style={regular polygon,regular polygon sides=4}]
      \node[shape=rectangle,draw=black,minimum size=0.75*\sizeforhopt] (A1) at (0,0) {};
      \node[shape=rectangle,draw=black,minimum size=0.75*\sizeforhopt] (A2) at ([shift={(0,-\yshifta)}]A1) {};
      \node[shape=circle,draw=black,minimum size=0.75*\sizeforhopt] (P1) at ([shift={(\xshift,0)}]A1) {};
      \node[shape=circle,draw=black,minimum size=0.75*\sizeforhopt] (P2) at ([shift={(\xshift,0)}]A2) {};
      \node[shape=rectangle,draw=black,minimum size=0.75*\sizeforhopt] (A3) at ([shift={(\xshift,+\yshiftb)}]P1) {};
      \node[shape=rectangle,draw=black,minimum size=0.75*\sizeforhopt] (A4) at ([shift={(\xshift,-\yshiftb)}]P1) {};
      \node[shape=rectangle,draw=black,minimum size=0.75*\sizeforhopt] (A5) at ([shift={(\xshift,+\yshiftb)}]P2) {};
      \node[shape=rectangle,draw=black,minimum size=0.75*\sizeforhopt] (A6) at ([shift={(\xshift,-\yshiftb)}]P2) {};
      \node[shape=circle,draw=black,minimum size=0.75*\sizeforhopt] (G3) at ([shift={(2*\xshift,0)}]P1) {};
      \node[shape=circle,draw=black,minimum size=0.75*\sizeforhopt] (G4) at ([shift={(2*\xshift,0)}]P2) {};
      \node[shape=circle,draw=black,minimum size=0.6*\sizeforhopt] (G33) at (G3) {};
      \node[shape=circle,draw=black,minimum size=0.6*\sizeforhopt] (G44) at (G4) {};

      \node at (A1) {\textsize$a_1^{(2)}$};
      \node at (A2) {\textsize$a_2^{(2)}$};
      \node at (P1) {\textsize$p_1$};
      \node at (P2) {\textsize$p_2$};
      \node at (A3) {\textsize$a_3^{(2)}$};
      \node at (A4) {\textsize$a_4^{(2)}$};
      \node at (A5) {\textsize$a_5^{(2)}$};
      \node at (A6) {\textsize$a_6^{(2)}$};
      \node at (G3) {\textsize$g_3$};
      \node at (G4) {\textsize$g_4$};

      \path [-,draw=blue] (A1) edge (P1);

      \path [-,draw=blue] (A2) edge (P2);

      \path [-] (P1) edge (A3.west);

      \path [-] (P1) edge (A5.west);
      
      \path [-] (P2) edge (A4.west);
      
      \path [-] (P2) edge (A6.west);
      
      \path [-,draw=blue] (A3) edge (G3);
      
      \path [-,draw=blue] (A4) edge (G3);

      \path [-,draw=blue] (A5) edge (G4);

      \path [-,draw=blue] (A6) edge (G4);
    \end{tikzpicture}
    \caption{\slg{} of $P_2$.}
  \end{subfigure}
  \caption{\slg{} of problems used in the proof of Thm.~\ref{thm:h_plus_opt}. Black edges indicate preconditions and blue edges indicate add effects.}
  \label{fig:hopt}
\end{figure}
\begin{proof}
  Consider the two (delete free) planning problems $P_1 = \gen{P, A_1, s_0, G}$ and $P_2 = \gen{P, A_2, s_0, G}$ with $P=\set{p_1, p_2, g_3, g_4}$, $G=\set{g_3, g_4}$, $s_0 = \emptyset$ and action sets $A_1 = \seta{a_i^{(1)} \mid i=1,\ldots,6}$, $A_2 = \seta{a_i^{(2)} \mid i=1,\ldots,6}$ where actions have no delete effects and are defined by
{
  \small
  \begin{align*}
    \pre(a_1^{(1)}) &= \emptyset, &\add(a_1^{(1)}) &= \seta{p_1}, \\ 
    \pre(a_1^{(2)}) &= \emptyset, &\add(a_1^{(2)}) &= \seta{p_1}, \\
    \pre(a_2^{(1)}) &= \emptyset, &\add(a_2^{(1)}) &= \seta{p_2}, \\ 
    \pre(a_2^{(2)}) &= \emptyset, &\add(a_2^{(2)}) &= \seta{p_2}, \\
    \pre(a_3^{(1)}) &= \seta{p_1}, &\add(a_3^{(1)}) &= \seta{g_3}, \\ 
    \pre(a_3^{(2)}) &= \seta{p_1}, &\add(a_3^{(2)}) &= \seta{g_3}, \\
    \pre(a_4^{(1)}) &= \seta{p_1}, &\add(a_4^{(1)}) &= \seta{g_3}, \\ 
    \pre(a_4^{(2)}) &= \seta{p_1}, &\add(a_4^{(2)}) &= \seta{g_4}, \\
    \pre(a_5^{(1)}) &= \seta{p_2}, &\add(a_5^{(1)}) &= \seta{g_4}, \\ 
    \pre(a_5^{(2)}) &= \seta{p_2}, &\add(a_5^{(2)}) &= \seta{g_3}, \\
    \pre(a_6^{(1)}) &= \seta{p_2}, &\add(a_6^{(1)}) &= \seta{g_4}, \\ 
    \pre(a_6^{(2)}) &= \seta{p_2}, &\add(a_6^{(2)}) &= \seta{g_4}.
   \end{align*}
}

  We have that the minimum plan cost for $P_1$ is 4 by applying actions $a_1^{(1)}, a_2^{(1)}, a_3^{(1)}, a_5^{(1)}$ whereas the minimum plan cost for $P_2$ is 3 with actions $a_1^{(1)}, a_3^{(1)}, a_5^{(1)}$, as seen in Fig.~\ref{fig:hopt}. Both $\hplus$ and $\hopt$ return 4 for $P_1$ and 3 for $P_2$.

  Colour refinement assigns the same invariant to the graph representations of $P_1$ and $P_2$ and thus by the contrapositive of \cite[Lem.~2]{xu2018powerful}, any MPNN assigns the same embedding to both graphs. 
\end{proof}

\begin{theorem}[MPNNs cannot learn any approximation of $h^*$]\label{thm:approx}
  Let $\Graph \in \set{\slg, \flg, \llg}$ and $c > 0$. There exists a pair of planning tasks $\Pi_1$ and $\Pi_2$ with $h(\Pi_1) \not= h(\Pi_2)$ such that for any set of parameters $\Theta$ for an MPNN we do not have $\abs{\F_{\Theta}(\Graph(\Pi_1)) - h(\Pi_1)}\leq c \wedge \abs{\F_{\Theta}(\Graph(\Pi_2)) - h(\Pi_2)}\leq c$. Furthermore, for any set of parameters we do not have $\abs{1 - \frac{\F_{\Theta}(\Graph(\Pi_1))}{h(\Pi_1)}}\leq c \wedge \abs{1 - \frac{\F_{\Theta}(\Graph(\Pi_2))}{h(\Pi_2)}}\leq c$.
\end{theorem}
\input{figures/hopt_approx.tex}
\begin{proof}
  Let us fix $n \in \N$ with $n>2$. Then we will construct a pair of planning problems whose optimal plan costs are $2n-1$ and $n^2$ respectively but are indistinguishable by MPNNs by any graph representations $\Graph \in \set{\slg, \flg, \llg}$ of the problems. Thus, we can make our absolute and relative errors, given by $n^2 - 2n + 1$ and $\frac{n^2}{2n-1}$ respectively, arbitrary large.

  Consider the two (delete free) planning problems given by $P_1 = \gen{P,A_1,s_0,G}$ and $P_2=\gen{O,A_2,s_0,G}$ with $P=\set{p(x,y)\mid x,y \in [n]}$, $G = \set{p(n,y)\mid y \in [n]} \subset P$, $s_0 = \emptyset$ and actions $A_1 = \set{a_1(y,z) \mid y,z \in [n]} \cup A$ and $A_2 = \set{a_2(y,z) \mid y,z \in [n]} \cup A$ where $A = \set{a(x,y) \mid x \in [n-1], y \in [n]}$. All actions have no delete effects and their preconditions and add effects are given as follows
{  
  \small
  \begin{align*}
    \pre(a(1,y)) &= \emptyset, &\add(a(1,y)) &= \set{p(1,y)}, \tag*{$\forall y \in [n]$} \\
    \pre(a(x,y)) &= \set{p(x-1,y)}, &\add(a(x,y)) &= \set{p(x,y)}, \tag*{$\forall x \in[2..n-1], y \in [n]$} \\
    \pre(a_1(y,z)) &= \set{p(n-1,y)}, &\add(a_1(y,z)) &= \set{p(n,y)}, \tag*{$\forall y,z \in [n]$} \\
    \pre(a_2(y,z)) &= \set{p(n-1,z)}, &\add(a_2(y,z)) &= \set{p(n,y)}, \tag*{$\forall y,z \in [n]$}
  \end{align*}
}
  where we note that the case $n=2$ is given in the proof of Thm.~\ref{thm:h_plus_opt}. We refer to Fig.~\ref{fig:large_counterexample} for the case of $n=3$. An optimal plan for $P_1$ consists of executing all actions $a \in A$ and $a_1(y,1)$ for $y \in [n]$. On the other hand, an optimal plan for $P_2$ consists only of executing $a(x,1)$ for $x \in [n-1]$ followed by $a_2(y,1)$ for all $y \in [n]$. Thus, the optimal plan costs for $P_1$ and $P_2$ are $n^2$ and $2n-1$ respectively. 
  
  As in the previous proof, any graph representations of the pair of problems for any $n$ are indistinguishable by colour refinement and hence by MPNNs~\cite[Lem.~2]{xu2018powerful}.
\end{proof}

\end{document}